\documentclass[review]{elsarticle}

\journal{ArXiv}

\usepackage{amstext,amssymb,amsmath,epsf,graphics}
\usepackage{xcolor}
\definecolor{red}{rgb}{1,0,0}

\definecolor{blue}{rgb}{0,0,1}

\definecolor{black}{rgb}{0,0,0}

\newtheorem{thm}{Theorem}
\newtheorem{lemma}{Lemma}

\newcommand{\elementwise}[1]{\mbox{\scriptsize$\circ{#1}$}}

\newcommand{\kmeas}[1]{\mbox{$|\!\!\wr{#1}\wr\!\!|_k$}}
\newcommand{\kmeask}[1]{\mbox{$|\!\!\wr{#1}\wr\!\!|_k^k$}}

\newcommand{\diag}{\mbox{${\rm diag}$}}

\newcommand{\x}{\mbox{${\bf x}$}}

\newcommand{\sgn}{\mbox{${\rm sgn}$}}

\renewcommand{\geq}{\geqslant}
\renewcommand{\leq}{\leqslant}

\newcommand{\bm}[1]{\boldsymbol #1}

\newcommand{\btheta}{\boldsymbol \theta}

\renewcommand{\Box}{\hspace*{\fill}~$\square$ \\}

\newcommand{\balpha}{\boldsymbol{\alpha}}
\newcommand{\bbeta}{\boldsymbol{\beta}}

\newcommand{\Real}{\mbox{$\mathbb{R}$}}

\newcommand{\Xpower}{|{\bf X}^T|^{\elementwise{\frac{1}{k-1}}}}

\begin{document}

\begin{frontmatter}

\title{Deterministic Bridge Regression for Compressive Classification}


\author[mymainaddress]{Kar-Ann Toh\corref{mycorrespondingauthor}}
\cortext[mycorrespondingauthor]{Corresponding author}
\ead{katoh@yonsei.ac.kr}

\author[mysecondaryaddress]{Giuseppe Molteni}
\ead{giuseppe.molteni1@unimi.it}

\author[mythirdaddress]{Zhiping Lin}
\ead{ezplin@ntu.edu.s}

\address[mymainaddress]{Room C216, Building 123, School of Electrical and Electronic Engineering, Yonsei University, 50 Yonsei-ro, Seodaemun-gu, Seoul 03722, Korea}
\address[mysecondaryaddress]{Dipartimento di Matematica, Universit\`{a} degli Studi di Milano, Via Saldini 50, 20133 Milano - Italy}
\address[mythirdaddress]{School of Electrical and Electronic Engineering, Nanyang Technological University, Singapore 639798}

\begin{abstract}
	Pattern classification with compact representation is an important component in machine intelligence. In this work, an analytic bridge solution is proposed for compressive classification. The proposal has been based upon solving a penalized error formulation utilizing an approximated $\ell_p$-norm. The solution comes in a primal form for over-determined systems and in a dual form for under-determined systems. While the primal form is suitable for problems of low dimension with large data samples, the dual form is suitable for problems of high dimension but with a small number of data samples. The solution has also been extended for problems with multiple classification outputs. Numerical studies based on simulated and real-world data validated the effectiveness of the proposed solution. 
\end{abstract}


\end{frontmatter}


\section{Introduction}

While high dimensionality of data provides rich feature information, the accompanying challenges along the classifier learning process cannot be ignored. Apart from the computational effort, an appropriate selection or a weighting between discriminative features and non-discriminative ones for accurate prediction is nontrivial. The {\em ordinary least squares} (ols) \cite{Duda1,Hastie1} regression, which minimizes the residual sum of squared errors, provides an unbiased estimation. However, the estimation comes with large variance when the input features have collinearity. Moreover, for under-determined systems, the estimation encounters singularity. The {\em ridge} regression \cite{Hoerl1,Hoerl2} circumvents the problem of singularity by shrinking the estimator via penalizing the weight coefficients in $\ell_2$-norm during minimization. Since then, the penalized models have evolved towards dealing with challenges related to features weighting and selection.    

Instead of utilizing all the predictors which might contain irrelevant ones,
the {\em least absolute shrinkage and selection operator} (lasso) \cite{Tibshirani1} shrinks the estimator with some parameters zero based on the $\ell_1$-norm penalty. This turns the challenge into a feature selection one. In order to encompass both capabilities of variable selection and variable shrinkage, the  {\em elastic-net} \cite{Zou1} leverages amidst ridge and lasso by weighting between the $\ell_1$-norm and the $\ell_2$-norm penalties.  Different from the elastic-net, the bridge regression \cite{Frank1} penalizes the sum of squared errors by the $\ell_p$-norm of weight coefficients. It does variable selection when $0<p\leq1$, and shrinks the coefficients when $p>1$. For $1<p<2$, the {\em bridge} regression shrinks the coefficients unevenly with a higher penalty to those less relevant ones. Attributed to the general penalty form of $\ell_p$-norm, the bridge regression fits well situations when it needs variable selection, weighting or when there exists collinearity.  

In recognition applications of computer vision and image processing related to artificial intelligence, problems with high input dimension and multiple outputs are almost unavoidable. Attributed to the pervasive coverage of social media, a large amount of data samples is deemed available for many applications. However, there are also situations when the data samples are scarce. For examples, scarce images of rare disease \cite{Julia1}, archaeological samples \cite{Irmgard1} and other objects may cast difficulties in effective learning. Several approaches are available to deal with this situation. These approaches include reduction of model complexity \cite{Krishnakumar1,ShuoX1,Fujiwara1,Fan1,Zou2}, data augmentation \cite{Niharika1}, and transfer learning \cite{WangLeye1,Zhou1}. Depending on the availability of supplementary information and the requirement of applications, each approach has its strengths and limitations. The approaches by data augmentation and transfer learning might provide accurate prediction. However, their successful adoption is highly hinged upon matching of the distribution of the augmented, generated or pretrained data with respect to that of the unknown ground truth data apart from the computational effort on data generation. The approach by reduction of model complexity offers a simplified model but might face the bottleneck when the data for learning is not representative. In this article, we shall utilize only the given data and work on the approach of model complexity reduction to suppress non-informative variables. In view of the lack of a deterministic solution for analysis and the outlook for an exactly converged recursive form for online applications, we thereby propose a proximal solution in analytic form to bridge regression for compressive classification. Different from \cite{Krishnakumar1,ShuoX1,Fujiwara1,Fan1,Zou2} which tackled the $\ell_0$, $\ell_1$ and $\ell_p$ formulations via an iterative search, a novel deterministic, non-iterative solution and algorithm has been proposed for solving the bridge regression which utilizes the $\ell_p$ norm penalty on weight coefficients with the error cost function.

The main contributions of this work are enumerated as follows:
\begin{itemize}
	\item Based on an approximation to the $\ell_p$-norm, an analytic solution for the bridge regression has been derived independently in primal form for over-determined systems, and in dual form for under-determined systems. The solution in primal form is found to have the same expression as that obtained based on a local quadratic approximation. 
	The solution in dual form does not find any precedent in the literature as we find no attempt from this perspective. This formulation can be useful for few-shots learning when data is scarce. 
	\item The analytic solution in primal and dual forms are extended to solve for problems with multiple outputs. In pattern classification, this formulation is useful for multiple category prediction. Although one could stack multiple predictors of single-output for multi-outputs prediction, the proposed solution for multi-outputs has the advantage of having a common covariance for output alignment. 
	\item An algorithm that packs the two analytic solutions for multiple outputs under a single estimation framework.
	An extensive study has been performed based on both simulated data and real-world data sets. Both the solution of the primal form and the dual form show stretchable weight coefficients estimation for different penalty settings. The solution in dual form shows higher trade-off between prediction accuracy and coefficient sparseness than that of solution in primal form.
\end{itemize}
The significance of this research outcome for the field of Artificial Intelligence (AI) lies in the establishment, for the first time in the literature, of a set of solutions for bridge regression in deterministic form that is useful for analysis and learning recognition applications. Such a deterministic form not only guarantees the convergence of solution but also renders computational efficiency i.e., without needing an iterative search for the solution. This opens up the feasibility for future development of a convergent bridge solution for online compressive learning. An online solution is an important component for real-time applications in the current AI world.

\section{Preliminaries}

Given a data set $\{\bm{x}_i, {y}_i\}_i^M$ of $M$ samples, an ols regression utilizing the model
\begin{equation}
	y_i = \bm{x}^T_i\balpha + \epsilon_i, \ \ \balpha\in\Real^D
\end{equation}
which minimizes the sum of squared errors $\sum_{i=1}^{M} \epsilon_i^2 = \sum_{i=1}^{M}(y_i - \bm{x}^T_i\balpha)^2$ can be performed with an optimal solution (least squared errors) given by $\hat{\balpha} =({\bf X}^T{\bf X})^{-1}{\bf X}^T{\bf y}$ where ${\bf X}=[\bm{x}_1,...,\bm{x}_M]^T$ and ${\bf y}=[y_1,...,y_M]^T$ for $M \geq D$ when ${\bf X}^T{\bf X}\in\Real^{D\times D}$ has full rank. For the situation when $M < D$ and when ${\bf X}{\bf X}^T\in\Real^{M\times M}$ has full rank, the solution given by $\hat{\balpha} ={\bf X}^T\left({\bf X}{\bf X}^T\right)^{-1}{\bf y}$ is known as the {\em least norm} solution, which is exact. In this work, we shall pack them together and call it the ols solution as follows:
\begin{equation}
	\hat{\balpha} = \left\{ \begin{array}{lll}
		\left({\bf X}^T{\bf X}\right)^{-1}{\bf X}^T{\bf y}, & M \geq D & {\rm (primal\ form)}\\
		{\bf X}^T\left({\bf X}{\bf X}^T\right)^{-1}{\bf y}, & M < D & {\rm (dual\ form)}
	\end{array}\right. .
\end{equation}

The {\em ridge} regression regularizes the ols learning by inclusion of a penalty to the weight coefficients $\balpha$ (also known as learning parameters) based on the $\ell_2$-norm. For $\lambda>0$, the resulted solution for minimizing $\sum_{i=1}^{M}(y_i - \bm{x}^T_i\balpha)^2 + \lambda\sum_{j=1}^{D}\alpha_j^2$ can be written as
\begin{equation}
	\hat{\balpha} = \left\{ \begin{array}{lll} 
		\left({\bf X}^T{\bf X} + \lambda{\bf I}\right)^{-1}{\bf X}^T{\bf y}, & M \geq D  & {\rm (primal\ form)}\\
		{\bf X}^T\left({\bf X}{\bf X}^T + \lambda{\bf I}\right)^{-1}{\bf y}, & M < D  & {\rm (dual\ form)}
	\end{array}\right. ,
\end{equation}
where ${\bf I}$ is an identity matrix congruence to its summing term. Effectively, the ridge regression provides an estimation with a set of shrunk weight coefficients relative to that of the ols. Due to the utilization of the $\ell_2$-norm penalty, the shrinkage is uniform in each of the $D$ dimension.

The {\em bridge} regression generalizes the ridge regression by replacing the $\ell_2$-norm penalty with an $\ell_p$-norm penalty of the weight coefficients where the range of $p$ is commonly taken within $0< p\leq 2$. In other words, the bridge regression uses the following cost function for minimization:
\begin{equation}
	\sum_{i=1}^{M}(y_i - \bm{x}^T_i\balpha)^2 + \lambda\sum_{j=1}^{D}|\alpha_j|^p .
\end{equation}
Due to the difficulty in dealing with the absolute operator and the nonlinear formulation, an analytic or closed-form solution is yet to be available. Moreover, separate treatments according to the under-determined and over-determined scenarios are not available in the literature. The formulation for the under-determined scenario is particularly useful when the data is scarce. This was termed {\em small sample size} (SSS) problem \cite{Jain12,LuJW2} before the deep learning era and is also known as {\em few-shot learning} \cite{WangYq1} in the current literature.

The bridge regression was first seen in \cite{Frank1} where several statistical tools for chemometrics regression were studied. According to the study, the parameter $p$ can be viewed as the degree to which the prior probability is concentrated along the favored directions. A value of $p\rightarrow 0$ places the prior mass towards the directions of the coordinate axes, expressing the prior belief that only a few of the predictor variables are likely to have high relative influence on the response. The structure of bridge was studied in \cite{FuWJ1} where a general approach to solve the bridge regression for $p\geq 1$ was developed. The algorithm, which was based on a modified Newton-Raphson method, solved iteratively for the unique solution for bridge for $p\geq 1$. According to \cite{Fan1}, the solution for bridge is continuous only when $p\geq 1$. In their proposal, a local quadratic approximation has been adopted iteratively for the $\ell_p$ penalized likelihood. It turned out that the minimization problem can be reduced to a quadratic minimization problem where the Newton-Raphson algorithm can be adopted to search for a solution \cite{Fan1}. 
In \cite{Zou2}, a local linear approximation has been proposed to replace the local quadratic approximation for solving the penalized likelihood with better computational efficiency. In \cite{ParkCW1}, both the local linear and local quadratic approximations have been studied. They showed that the bridge estimator is a robust choice under various circumstances comparing with ridge, lasso, and elastic net.

\section{Proximal Bridge Regression}

In this section, we propose an analytic solution for an approximated bridge regression called {\em proximal bridge} regression. The solution comes in primal form for over-determined systems and in dual form for under-determined systems. We shall introduce an approximation to the $\ell_p$-norm before presenting the two solution forms.

\subsection{A $k$-measure for $\ell_p$-norm approximation}

Consider a positive valued penalty term that is an approximation of the $\ell_p$-norm, in
which the absolute value operator is replaced by a differentiable function
$f_{\epsilon}$:
\begin{equation}\label{eqn_k_norm}
	\kmeas{\bm{\alpha}} := \left(\sum^{D}_{j=1}f_{\epsilon}(\alpha_j)^{k}
	\right)^{1/k} \negthinspace\negthinspace ,
\end{equation}
where $\balpha=[\alpha_1,...,\alpha_D]^T$ is a parameter vector.
Here, the power term $k$ replaces $p$ in the $\ell_p$-norm to indicate the approximation. A convenient
choice for approximating the absolute operator, which can be efficiently computed, is
$f_{\epsilon}(\alpha_j)=\sqrt{\alpha_j^2+\epsilon} \approx |\alpha_j|$, $\epsilon>0$ (see \cite{Toh96,Ramirez1}) and
Fig.~\ref{fig_abs_approx}). Note that $\lim_{\epsilon\rightarrow 0} f_{\epsilon}(\cdot) = |\cdot|$
for arbitrary $\epsilon>0$. For finite $\epsilon$ values, the function $\kmeas{\bm{\alpha}}$ is
not a norm because it does not have the absolute homogeneity property. We shall call
$\kmeas{\cdot}$ \eqref{eqn_k_norm} a \emph{$k$-measure} operator for convenience hereon.

\begin{figure}[hhh]
	\begin{center}
		\epsfxsize=6.8cm 
		\epsffile[76   199   546   606]{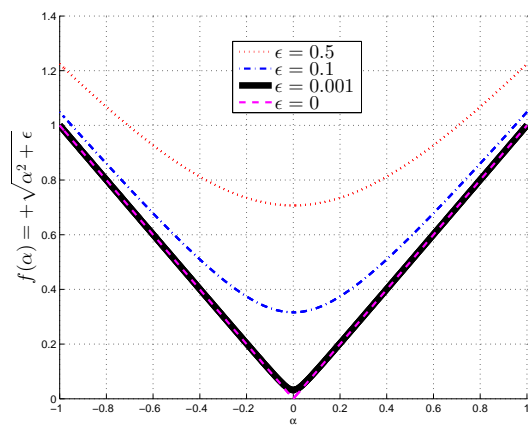}
		\caption{Plot of $f(\alpha)=\sqrt{\alpha^2+\epsilon}$ at several $\epsilon$ values.}
		\label{fig_abs_approx}
	\end{center}
\end{figure}

In the following, the raised power form of $k$-measure is shown to be convex when the
approximation function $f_{\epsilon}(\cdot)$ is convex.\\
\begin{lemma} \label{lemma_convex}
	$\kmeask{\balpha}$ is convex on $\balpha$ when $f_{\epsilon}$ is convex for all $k\geq 1$.
\end{lemma}
\begin{proof}
	Based on the convexity of $f_{\epsilon}$ on each element $\alpha_i$, $i=1,...,D$ (of the
	parameter vector $\balpha$), we have
	\begin{equation}\label{eqn_convexity_of_f}
		f_{\epsilon}(\lambda\alpha_{i1}+(1-\lambda)\alpha_{i2})
		\leq \lambda f_{\epsilon}(\alpha_{i1}) + (1-\lambda) f_{\epsilon}(\alpha_{i2}),
		\ \ \ 0\leq\lambda\leq 1.
	\end{equation}
	Suppose $h(\theta):=\theta^{k}$ with $k,\theta>0$ where we know that $h$ is nondecreasing
	and convex on $\theta$ since $dh/d\theta=k\theta^{k-1}\geq 0$ and
	$d^2h/d\theta^2=k(k-1)\theta^{k-2}\geq 0$, $\forall k\geq 1$ (see also \cite{Hardy1}).
	Using \eqref{eqn_convexity_of_f} plus the fact that $h$ is nondecreasing and convex, we
	have for each $i=1,...,D$,
	\begin{eqnarray}
		h(f_{\epsilon}(\lambda\alpha_{i1}+(1-\lambda)\alpha_{i2}))
		&\leq & h(\lambda f_{\epsilon}(\alpha_{i1}) + (1-\lambda) f_{\epsilon}(\alpha_{i2}))
		\nonumber \\
		&\leq & \lambda h(f_{\epsilon}(\alpha_{i1})) + (1-\lambda) h(f_{\epsilon}(\alpha_{i2})),
		\ \ \ 0\leq\lambda\leq 1 . \label{eqn_convexity_of_hf1}
	\end{eqnarray}
	Since summation of convex functions preserves the convexity, we have
	\begin{eqnarray}
		\sum^{D}_{i=1} h(f_{\epsilon}(\lambda\alpha_{i1}+(1-\lambda)\alpha_{i2}))
		\leq \sum^{D}_{i=1} \lambda h(f_{\epsilon}(\alpha_{i1})) + (1-\lambda) h(f_{\epsilon}(\alpha_{i2})),
		\ \ 0\leq\lambda\leq 1,
		\label{eqn_convexity_of_hf4}
	\end{eqnarray}
	which means convexity of $\sum^{D}_{i=1} h(\alpha_i) = \sum^{D}_{i=1} f_{\epsilon}(\alpha_i)^k =
	\kmeask{\balpha}$ on $\balpha=[\alpha_1,...,\alpha_{D}]^T$. 
\end{proof}

Fig.~\ref{fig_q_norm} shows the contours of the $\ell_{p}$-norm metric
for $1<p\leq 2$ and the corresponding ${k}$-measure ($\kmeas{\balpha}$,
\eqref{eqn_k_norm}) together with its $k$-powered form ($\kmeask{\balpha}$) within the
same interval. From the third row of plots in Fig.~\ref{fig_q_norm}, except for the difference
in curvature, we see that the entire ${k}$-measure and its $k$-powered form approximate
well to the solution $p$-norm for the plotted range of $1<\{p,k\}\leq 2$. 
This suggests vertices of the vector space being feasible solutions for the desired
constrained solution search. Such an observation shall be exploited in the following
development for compressive solution when $1<k\leq 2$.\\*[1mm]
\begin{figure}[hhh] 
	\begin{center}
		\epsfxsize=9.8cm
		\epsfysize=13.8cm
		\epsffile[46    50   504   638]{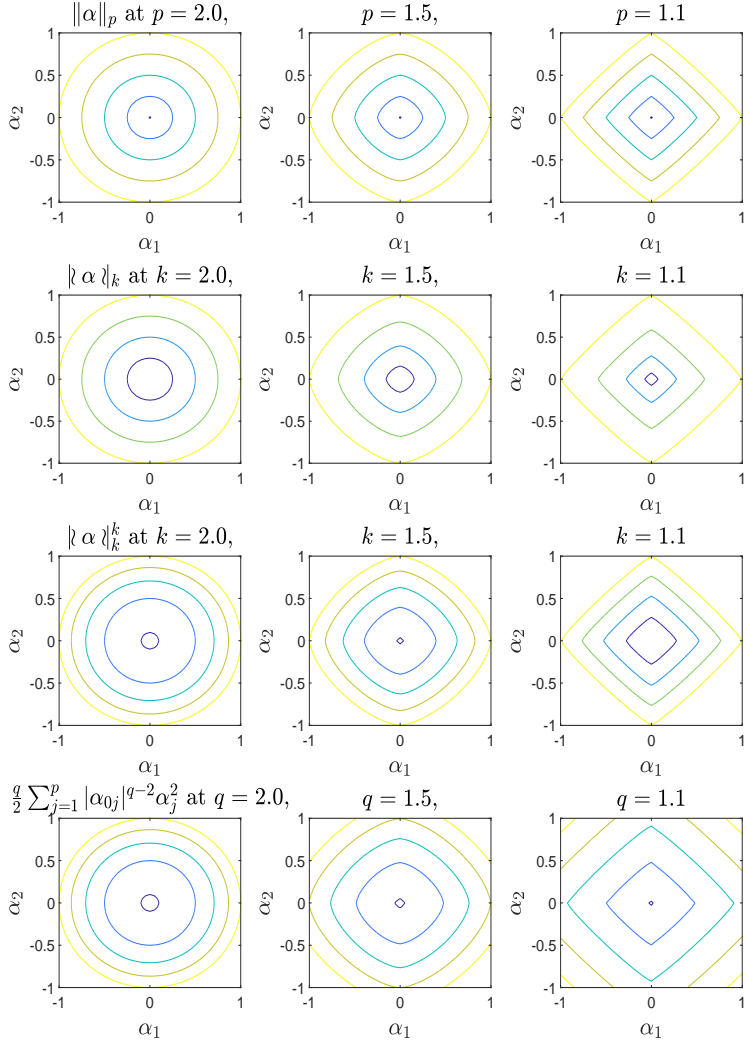}
		\caption{Contour plots at levels [0.1, 0.4, 0.7, 1]. Top row: $\|\bm{\alpha}\|_p$ for $p\in\{2.0,1.5,1.1\}$.
		Second row: $\kmeas{\bm{\alpha}}$ for $k\in\{2.0,1.5,1.1\}$ at $\epsilon=0.0001$.
		Third row: $\kmeask{\bm{\alpha}}$ for $k\in\{2.0,1.5,1.1\}$ at $\epsilon=0.0001$. Bottom row: $\frac{q}{2}\sum^p_{j=1}|\alpha_{0j}|^{q-2}\alpha_j^2$ for $q\in\{2.0,1.5,1.1\}$ with $\alpha_{0j}=\alpha_j-0.01$. }
		\label{fig_q_norm}
	\end{center}
\end{figure}

The bottom row of Fig.~\ref{fig_q_norm} shows another approximation of the $\ell_p$-norm by a Local Quadratic Approximation (LQA, \cite{Fan1}) given by
\begin{equation} \label{eqn_LQA_approx}
	|\alpha_j|^q_q \approx |\alpha_{0j}|^q + \frac{q}{2}\frac{|\alpha_{0j}|^{q-1}}{|\alpha_{0j}|}(\alpha_j^2 - \alpha_{0j}^2),
\end{equation}	
where the minimization problem of an approximated bridge regression can be expressed as
\begin{equation} \label{eqn_LQA_min}
	\arg\min_{\balpha}\left\{({\bf y}-{\bf X}\bm{\alpha})^T({\bf y}-{\bf X}\bm{\alpha}) + \frac{q}{2} \sum^D_{j=1}|\alpha_{0j}|^{q-2}\alpha_j^2 \right\}.
\end{equation}	
The plots of LQA in this row have been generated based on 
$\frac{q}{2}\sum^p_{j=1}|\alpha_{0j}|^{q-2}\alpha_j^2$ for $q\in\{2.0,1.5,1.1\}$ at $\alpha_{0j}=\alpha_j-0.01$.
\\


\subsection{Proximal bridge regression in primal form}

For over-determined systems, we minimize the sum of squared errors with a $k$-measure penalty as shown in Theorem~\ref{thm1} below. We call such minimization \emph{proximal bridge regression in primal form} (or \emph{primal $p$-bridge} in brief). In this formulation, $\circ$ denotes an element-wise operator. For example, ${\bf A}^{\elementwise{k}}$ indicates raising each element of ${\bf A}$ to the power $k$. Also, $\diag(\bm{a})$ denotes a diagonal matrix with its diagonal elements given by vector $\bm{a}$, and ${\rm eig}_j({\bf A})$ denotes the $j$th eigenvalue of matrix ${\bf A}$. 

\begin{thm}\label{thm1}
	Given the data $\{ \bm{x}_i, {y}_i\}$, $i=1,...,M$ where $ \bm{x}_i=[x_{i,1},\cdots,x_{i,D}]^T$ and ${y}_i$ are respectively the regressors and the response for the $i$th observation. Consider the linear regression model ${\bf X}\bm{\alpha}$ with parameter vector $\bm{\alpha}\in\Real^D$ and regressor matrix ${\bf X}=[\bm{x}_1,\cdots,\bm{x}_M]^T$. Suppose ${\bf X}^T{\bf X}$ is of full rank. Then, under the limiting case of ${\bm{\epsilon}\rightarrow{\bf 0}}$ and for $k\geq 1$, $\hat{\balpha}$ that satisfies
	\begin{eqnarray} 
		\balpha
		&=& \left(\frac{\lambda{k}}{2}\,\diag\{|\balpha|^{\elementwise{(k-2)}}\} + {\bf X}^T{\bf X}\right)^{-1}{\bf X}^T{\bf y}
		\label{eqn_qnorm_alpha_solution_thm1}
	\end{eqnarray}
	minimizes 
	\begin{equation}\label{eqn_q_norm_obj888}
		({\bf y}-{\bf X}\bm{\alpha})^T({\bf y}-{\bf X}\bm{\alpha}) + \lambda\kmeask{\balpha} ,
	\end{equation}
	when the matrix $(\frac{\lambda{k}}{2}\,\diag\{|\balpha|^{\elementwise{(k-2)}}\} + {\bf X}^T{\bf X})$ is invertible. This happens for sure as soon as 
	\begin{equation}\label{eqn_nonsingular_reg}
		\frac{\lambda{k}}{2} \max_{j\in\{1,...,D\}} (|\alpha_j|^{(k-2)})   < \min_{j\in\{1,...,D\}} ({\rm eig}_j ({\bf X}^T{\bf X})).
	\end{equation}
\end{thm}

\begin{proof} According to the definition of $k$-measure in \eqref{eqn_k_norm}, let
	$\bar{\bm{\alpha}}:=[(\alpha_0^2+\epsilon)^{{k}/4},\cdots,(\alpha_{D-1}^2+\epsilon)^{{k}/4}]^T$
	where we can write $\kmeas{\bm{\alpha}}=(\bar{\bm{\alpha}}^T\bar{\bm{\alpha}})^{1/k}$ and
	$\kmeask{\bm{\alpha}}=(\bar{\bm{\alpha}}^T\bar{\bm{\alpha}})$. 
	Next, take the first derivative of \eqref{eqn_q_norm_obj888} with respect to $\balpha$ and set it to zero:
	\begin{eqnarray}
		\frac{\partial}{\partial\balpha}\left(({\bf y}-{\bf X}\balpha)^T({\bf y}-{\bf X}\balpha) + \lambda   	\bar{\balpha}^T\bar{\balpha}\right)
		&=& {\bf 0} \nonumber \\
		-2{\bf X}^T({\bf y}-{\bf X}\balpha) + 
		\lambda\frac{k}{4}\cdot2\balpha\circ\left(\balpha^{\elementwise{2}}+\bm{\epsilon}\right)^{\elementwise{(\frac{k}{4}-1})}\circ 2\bar{\balpha}   
		&=& {\bf 0} \nonumber \\
		-2{\bf X}^T({\bf y}-{\bf X}\balpha) +
		\lambda{k}\,\balpha\circ\left(\balpha^{\elementwise{2}}+\bm{\epsilon}\right)^{\elementwise{(\frac{k}{4}-1})}\circ
		\left(\bm{\alpha}^{\elementwise{2}}+\bm{\epsilon}\right)^{\elementwise{\frac{k}{4}}} 
		&=& {\bf 0} \nonumber \\
		\Rightarrow\hspace{5mm} \lambda {k}\,\balpha\circ\left(\balpha^{\elementwise{2}}+\bm{\epsilon}\right)^{\elementwise{(\frac{k}{2}-1})} 
		&=& 2{\bf X}^T({\bf y}-{\bf X}\balpha).
		\label{eqn_qnorm_alpha550}
	\end{eqnarray}
	For the limiting case of $\bm{\epsilon}$, we have
	\begin{eqnarray}
		\lim_{\bm{\epsilon}\rightarrow{\bf 0}}{k}\,\balpha\circ\left(\balpha^{\elementwise{2}}+\bm{\epsilon}\right)^{\elementwise{(\frac{k}{2}-1})}
		&=& {k}\,\balpha\circ\left(\balpha^{\elementwise{2}}\right)^{\elementwise{(\frac{k}{2}-1})} \nonumber\\
		&=& {k}\,\sgn(\balpha)\circ(\balpha^{\elementwise{2}})^{\frac{1}{2}}\circ\left(\balpha^{\elementwise{2}}\right)^{\elementwise{\frac{k-2}{2}}} \nonumber\\
		&=& {k}\,\sgn(\balpha)\circ\left(\balpha^{\elementwise{2}}\right)^{\elementwise{\frac{k-1}{2}}} .
		\label{eqn_qnorm_alpha338}	
	\end{eqnarray}
	Equation~\eqref{eqn_qnorm_alpha550} can then be written as
	\begin{eqnarray}
		\lambda {k}\,\sgn(\balpha)\circ\left(\balpha^{\elementwise{2}}\right)^{\elementwise{\frac{k-1}{2}}} 
		&=& 2{\bf X}^T({\bf y} - {\bf X}\balpha)\nonumber\\
		\Rightarrow\ \frac{\lambda{k}}{2}\,\sgn(\balpha)\circ|\balpha|^{\elementwise{(k-1)}} 
		&=& {\bf X}^T({\bf y} - {\bf X}\balpha). \nonumber\\
		\frac{\lambda{k}}{2}\,\sgn(\balpha)\circ|\balpha|^{\elementwise{(k-1)}} + {\bf X}^T{\bf X}\balpha
		&=& {\bf X}^T{\bf y}, \nonumber\\
		\frac{\lambda{k}}{2}\,\balpha\circ|\balpha|^{\elementwise{(k-2)}} + {\bf X}^T{\bf X}\balpha
		&=& {\bf X}^T{\bf y}, \nonumber\\
		\left(\frac{\lambda{k}}{2}\,\diag\{|\balpha|^{\elementwise{(k-2)}}\} + {\bf X}^T{\bf X}\right)\balpha
		&=& {\bf X}^T{\bf y},
		\label{eqn_qnorm_alpha660}
	\end{eqnarray}
	which leads to \eqref{eqn_qnorm_alpha_solution_thm1} when $\left(\frac{\lambda{k}}{2}\,\diag\{|\balpha|^{\elementwise{(k-2)}}\} + {\bf X}^T{\bf X}\right)$ is nonsingular. 
	Let ${\bf A}=\frac{\lambda{k}}{2}\,\diag\{|\balpha|^{\elementwise{(k-2)}}\}$ and 
	${\bf B}={\bf X}^T{\bf X}$ which is given to be of full rank, then ${\bf A}+{\bf B}$ can be written as ${\bf B}({\bf I} + {\bf B}^{-1}{\bf A})$. Based on the power series expansion, we have $({\bf I} + {\bf B}^{-1}{\bf A})^{-1}= {\bf I} + (-{\bf B}^{-1}{\bf A}) + (-{\bf B}^{-1}{\bf A})^2 + (-{\bf B}^{-1}{\bf A})^3 + \cdots$, that actually converges to the inverse of ${\bf I}+{\bf B}^{-1}{\bf A}$ whenever $\|-{\bf B}^{-1}{\bf A}\| < 1$ for any sub-multiplicative norm.
	When ${\bf A}$ is symmetric, we have $||{\bf A}|| = \max_j(|{\rm eig}_j({\bf A})|)$ for matrix norm given by $||{\bf A}||  := \sup_{{\bf v} \neq 0}  \frac{||{\bf A}{\bf v}||}{||{\bf v}||}$. Since this norm is sub-multiplicative, we also have $\|-{\bf B}^{-1}{\bf A}\|\leq\|-{\bf B}^{-1}\|\hspace{-1mm}\cdot\hspace{-1mm}\|{\bf A}\|$. Hence, $\|-{\bf B}^{-1}{\bf A}\| < 1$ is implied by 
	\begin{equation} 
		\max_j (|{\rm eig}_j (({\bf X}^T{\bf X})^{-1})|)	\cdot \frac{\lambda{k}}{2} \max_j (|\alpha_j|^{(k-2)}) < 1 ,   
	\end{equation}
	or
	\begin{equation} 
		\frac{1}{\min_j (|{\rm eig}_j ({\bf X}^T{\bf X})|)}\cdot \frac{\lambda{k}}{2} \max_j (|\alpha_j|^{(k-2)}) < 1 .   
	\end{equation}
	This leads to \eqref{eqn_nonsingular_reg} where the absolute values are not necessary as eigenvalues for ${\bf X}^T{\bf X}$ are positive.

	Finally, as both $({\bf y}-{\bf X}\balpha)^T({\bf y}-{\bf X}\balpha)$ and $\kmeask{\balpha}$ (Lemma~\ref{lemma_convex}) are convex on $\balpha$ for $k\geq 1$, the summation of two convex functions in the objective function~\eqref{eqn_q_norm_obj888} is convex. Hence the minimizer.\\
\end{proof}

\noindent{\bf Remark 1: } 

It is interesting to observe that the solution given by \eqref{eqn_qnorm_alpha_solution_thm1} appears to have a similar form (see \eqref{eqn_qnorm_alpha_solution_Park}) as that in \cite{Fan1,ParkCW1} where a different $\ell_p$-norm approximation, based on the LQA (i.e., minimization of \eqref{eqn_LQA_min}) instead of the $k$-measure (i.e, minimization of \eqref{eqn_q_norm_obj888}), had been utilized for the penalty term: 
\begin{equation} \label{eqn_qnorm_alpha_solution_Park}
	\hat{\balpha}_j
	= \left(\frac{\lambda{q}}{2}\,\diag\{|\balpha_{0j}|^{\elementwise{(q-2)}}\} + {\bf X}^T{\bf X}\right)^{-1}{\bf X}^T{\bf y}.
\end{equation}

\hspace*{\fill}\Box

\subsection{Proximal bridge regression in dual form}

For under-determined systems, we minimize $\kmeask{\bm{\alpha}}$ subject to ${\bf y}={\bf X}\bm{\alpha}$. 
We call such minimization \emph{proximal bridge regression in dual form} (or \emph{dual $p$-bridge} in brief).
Similar to the primal proximal bridge regression, our goal here is to have a compressive estimate for $1<k\leq 2$.

\begin{thm}\label{thm2}
	Consider an under-determined system ${\bf y}={\bf X}\bm{\alpha}$, where ${\bf
		y}\in\Real^M$ is the given target vector, ${\bf X}\in\Real^{M\times D}$ is the regressor matrix and $\bm{\alpha}\in\Real^D$ is the parameter vector, with number of samples $M<D$ regressor dimensions.	
	Assume ${\bf X}{\bf X}^T$ and ${\bf X}\Xpower$ are of full rank for certain $k>1$. Then for that $k>1$ and under the limiting case of ${\bm{\epsilon}\rightarrow{\bf 0}}$, the stationary point given by
	\begin{eqnarray}
		\hat{\balpha} &=&
		\sgn\left(\btheta\right)\circ 
		\left|{\bf X}^T\left({\bf X}{\bf X}^T\right)^{-1}{\bf X}\btheta^{\elementwise{(k-1)}}\right|^{\elementwise{\frac{1}{k-1}}}  ,
		\label{eqn_qnorm_alpha_solution_dual_thm2}
	\end{eqnarray}
	where
	\begin{eqnarray}
		\btheta
		&=& \Xpower\left[{\bf X}\Xpower\right]^{-1} {\bf y} ,
		\label{eqn_qnorm_alpha_theta}
	\end{eqnarray}
	minimizes
	\begin{equation}\label{eqn_q_norm_obj111}
		\kmeask{\balpha}
		\ \,{\rm subject\ to\ \,} {\bf y}={\bf X}\balpha .
	\end{equation}
\end{thm}

\begin{proof}
	According to the definition of $k$-measure in \eqref{eqn_k_norm}, let
	$\bar{\bm{\alpha}}:=[(\alpha_0^2+\epsilon)^{{k}/4},\cdots,(\alpha_{D-1}^2+\epsilon)^{{k}/4}]^T$
	where we can write $\kmeas{\bm{\alpha}}=(\bar{\bm{\alpha}}^T\bar{\bm{\alpha}})^{1/k}$ and
	$\kmeask{\bm{\alpha}}=(\bar{\bm{\alpha}}^T\bar{\bm{\alpha}})$. Then, taking the first
	derivative of the Lagrange function of \eqref{eqn_q_norm_obj111} and setting it to zero
	gives:
	\begin{eqnarray}
		\frac{\partial}{\partial\balpha}\left(\bar{\balpha}^T\bar{\balpha}
		+ \bm{\beta}^T({\bf y}-{\bf X}\balpha)\right) &=& {\bf 0} \nonumber \\
		\frac{k}{4}\cdot2\balpha\circ\left(\balpha^{\elementwise{2}}+\bm{\epsilon}\right)^{\elementwise{(\frac{k}{4}-1})}\circ
		2\bar{\balpha} - {\bf X}^T\bbeta  &=& {\bf 0} \nonumber \\
		{k}\,\balpha\circ\left(\balpha^{\elementwise{2}}+\bm{\epsilon}\right)^{\elementwise{(\frac{k}{4}-1})}\circ
		\left(\bm{\alpha}^{\elementwise{2}}+\bm{\epsilon}\right)^{\elementwise{\frac{k}{4}}} - {\bf X}^T\bbeta  &=& {\bf 0} \nonumber \\
		\Rightarrow\hspace{5mm}{k}\,\balpha\circ\left(\balpha^{\elementwise{2}}+\bm{\epsilon}\right)^{\elementwise{(\frac{k}{2}-1})} &=& {\bf
			X}^T\bbeta.\hspace{6.8mm}
		\label{eqn_qnorm_alpha110}
	\end{eqnarray}
	For the limiting case of $\bm{\epsilon}$, we have
	\begin{eqnarray}
		\lim_{\bm{\epsilon}\rightarrow{\bf 0}}{k}\,\balpha\circ\left(\balpha^{\elementwise{2}}+\bm{\epsilon}\right)^{\elementwise{(\frac{k}{2}-1})}
		&=& {k}\,\balpha\circ\left(\balpha^{\elementwise{2}}\right)^{\elementwise{(\frac{k}{2}-1})},
	\end{eqnarray}
	which implies 
	\begin{eqnarray}
		{k}\,\balpha\circ\left(\balpha^{\elementwise{2}}\right)^{\elementwise{\frac{k-2}{2}}} &=& {\bf X}^T\bbeta \nonumber \\
		{k}\,\sgn(\balpha)\circ(\balpha^{\elementwise{2}})^{\frac{1}{2}}\circ\left(\balpha^{\elementwise{2}}\right)^{\elementwise{\frac{k-2}{2}}} &=& {\bf X}^T\bm{\beta} \nonumber \\
		{k}\,\sgn(\balpha)\circ\left(\balpha^{\elementwise{2}}\right)^{\elementwise{\frac{k-1}{2}}} &=& {\bf X}^T\bm{\beta} \nonumber \\
		&& \hspace{-4.3cm}\Rightarrow\hspace{5mm}\left(\balpha^{\elementwise{2}}\right)^{\elementwise{\frac{k-1}{2}}}\ \ =\ \ \sgn(\balpha)\circ\left(\frac{1}{k}{\bf X}^T\bm{\beta}\right)
		.\label{eqn_alpha_square}
	\end{eqnarray}
	Taking square elementwise for both sides of \eqref{eqn_alpha_square}, we have
	\begin{eqnarray}
		\left(\bm{\alpha}^{\elementwise{2}}\right)^{\elementwise{(k-1)}} &=& \left(\frac{1}{k}{\bf X}^T\bm{\beta}\right)^{\elementwise{2}}
		.
	\end{eqnarray}
	We know that the vector
	$\lim_{\bm{\epsilon}\rightarrow{\bf 0}}(\balpha^{\elementwise{2}}+\bm{\epsilon})$ has
	nonnegative elements and thus $\lim_{\bm{\epsilon}\rightarrow{\bf
			0}}\left(\balpha^{\elementwise{2}}+\bm{\epsilon}\right)^{\elementwise{(\frac{k}{2}-1)}}$
	has nonnegative elements. Hence, we deduce from \eqref{eqn_qnorm_alpha110} that 
	$\sgn(\balpha)=\sgn({\bf X}^T\bm{\beta})$, and
	\begin{eqnarray}
		\balpha &=&
		\sgn({\bf X}^T\bbeta)\circ
		\left|{\bf X}^T\left\{\frac{1}{k}\bm{\beta}\right\}\right|^{\elementwise{\frac{1}{k-1}}}
		.
		\label{eqn_qnorm_alpha111}
	\end{eqnarray}
	Next, suppose that
	\begin{eqnarray}
		\sgn({\bf X}^T\bbeta)\circ
		\left|{\bf X}^T\left\{\frac{1}{k}\bm{\beta}\right\}\right|^{\elementwise{\frac{1}{k-1}}}
		&=&  \Xpower  \bm{\gamma}
		\label{eqn_qnorm_alpha_gamma}
	\end{eqnarray}
	for some $\bm{\gamma}$, then premultiply ${\bf X}$ to both sides of \eqref{eqn_qnorm_alpha111} gives
	\begin{eqnarray}
		{\bf X}\balpha &=& {\bf X}\,\sgn({\bf X}^T\bbeta)\circ
		\left|{\bf X}^T\left\{\frac{1}{k}\bm{\beta}\right\}\right|^{\elementwise{\frac{1}{k-1}}} \nonumber\\
		\Rightarrow
		{\bf X}\balpha &=& {\bf X}\Xpower\bm{\gamma},\ \ {\rm according\ to\ \eqref{eqn_qnorm_alpha_gamma}} \nonumber\\
		\Rightarrow
		{\bf y} &=& {\bf X}\Xpower\bm{\gamma},\ \ {\rm since\ }{\bf y}={\bf X}\balpha \nonumber\\
		\Rightarrow 
		\bm{\gamma} &=& \left[{\bf X}\Xpower\right]^{-1} {\bf y},\ 	\ {\rm since\ } \left[{\bf X}\Xpower\right]^{-1}
		{\rm is\ invertible} .
		\label{eqn_qnorm_gamma1}
	\end{eqnarray}
	Knowing also that ${\bf X}{\bf X}^T$ is invertible, we substitute \eqref{eqn_qnorm_gamma1} into \eqref{eqn_qnorm_alpha_gamma} and get
	\begin{eqnarray}
		\sgn({\bf X}^T\bbeta)\circ
		\left|{\bf X}^T\left\{\frac{1}{k}\bm{\beta}\right\}\right|^{\elementwise{\frac{1}{k-1}}}
		&=& \Xpower\left[{\bf X}\Xpower\right]^{-1} {\bf y} \nonumber\\
		{\bf X}^T\left\{\frac{1}{k}\bbeta\right\}
		&=& \sgn({\bf X}^T\bbeta)\circ
		\left|\Xpower\left[{\bf X}\Xpower\right]^{-1} {\bf y} \right|^{\elementwise{(k-1)}}\nonumber\\
		{\bf X}{\bf X}^T\left\{\frac{1}{k}\bbeta\right\}
		&=& \sgn({\bf X}^T\bbeta)\circ 
		{\bf X}\left|\Xpower\left[{\bf X}\Xpower\right]^{-1} {\bf y} \right|^{\elementwise{(k-1)}}\nonumber\\
		\left\{\frac{1}{k}\bbeta\right\}
		&=& \sgn({\bf X}^T\bbeta)\circ 
		\left({\bf X}{\bf X}^T\right)^{-1}{\bf X}\left|\Xpower\left[{\bf X}\Xpower\right]^{-1} {\bf y} \right|^{\elementwise{(k-1)}}
		\!\!\!\!.
		\label{eqn_qnorm_alpha_gamma2}
	\end{eqnarray}
	Subsequently, substitute \eqref{eqn_qnorm_alpha_gamma2} into \eqref{eqn_qnorm_alpha111} and we have
	\begin{eqnarray}
		\hat{\balpha} &=&
		\sgn({\bf X}^T\bbeta)\circ
		\left|{\bf X}^T\left({\bf X}{\bf X}^T\right)^{-1}{\bf X}\left(\Xpower\left[{\bf X}\Xpower\right]^{-1} {\bf y} \right)^{\elementwise{(k-1)}}\right|^{\elementwise{\frac{1}{k-1}}} \nonumber\\
		&=&
		\sgn\left(\btheta\right)\circ
		\left|{\bf X}^T\left({\bf X}{\bf X}^T\right)^{-1}{\bf X}\btheta^{\elementwise{(k-1)}}\right|^{\elementwise{\frac{1}{k-1}}} 
		,
		\label{eqn_qnorm_alpha333}
	\end{eqnarray}
	where  
	\begin{eqnarray}
		\btheta
		&=& \Xpower\left[{\bf X}\Xpower\right]^{-1} {\bf y} .
		\label{eqn_qnorm_gamma2}
	\end{eqnarray}
	The sign of $\sgn({\bf X}^T\bbeta)=\sgn\left(\btheta\right)$ has been deduced from the top row of \eqref{eqn_qnorm_alpha_gamma2}. Equations~\eqref{eqn_qnorm_alpha333}-\eqref{eqn_qnorm_gamma2} hold well without singularity for all $\btheta\in\Real^D$ and all ${\bf X}\in\Real^{M\times D}$ for $k>1$.
	Finally, since $\kmeask{\balpha}$ is convex according to Lemma~\ref{lemma_convex} and the linear constraint function (${\bf y}={\bf X}\balpha$) is also convex, the Lagrange function of \eqref{eqn_q_norm_obj111} is convex. Hence the minimizer.
	
\end{proof}

\noindent{\bf Remark 2:}

\indent In \cite{Toh96}, a simpler version of global solution was conjectured for a weighted least norm regression. However, Theorem~\ref{thm2} reveals that such a solution cannot be any simpler. From application viewpoint, although the validity of $k$ stretches beyond 2 in Theorem~\ref{thm1} and in Theorem~\ref{thm2}, the region of interest for parametric shrinking is $k<2$. We shall thus focus on $k\in[1,2]$ for over-determined systems and $k\in(1,2]$ for under-determined systems in our development, both included the well-known non-compressive $\ell_2$-norm for benchmarking purpose. 
\hspace*{\fill}\Box

\subsection{Extension to Multiple Outputs} \label{sec_multiclass}

The above results can be extended for regression with multiple outputs. Particularly, by utilizing the same regressor matrix (${\bf X}$) with different outputs (${\bf y}_l$, $l=1,...,C$), the solution can be stacked for concurrent prediction. For example, suppose $\hat{\bm{A}}$ stacks the multiple columns of estimated coefficient vectors $[\hat{\balpha}_1,...,\hat{\balpha}_C]\in\Real^{D\times C}$, then the prediction can be computed as $\hat{\bf Y}={\bf X}\hat{\bm{A}}$. For classification applications, an one-hot encoding can be adopted for learning and the winner-take-all technique can be used to predict the outcome from the multi-category responses. 
As the outputs do not depend on each other, the extensions for the primal and the dual forms are straightforward.

\begin{thm}\label{thm3}
		Given the data $\{ \bm{x}_i, {y}_{i,l}\}$, $i=1,...,M,\ l=1,...,C$ where $ \bm{x}_i=[x_{i,1},\cdots,x_{i,D}]^T$ and ${y}_{i,l}$ are respectively the regressors and the response for the $i$th observation of the $l$th output. Consider the linear regression model ${\bf X}\bm{A}$ with parameter matrix $\bm{A}=[\balpha_1,...,\balpha_C]\in\Real^{D\times C}$ and regressor matrix ${\bf X}=[\bm{x}_1,\cdots,\bm{x}_M]^T$. Suppose ${\bf X}^T{\bf X}$ is of full rank. Then, under the limiting case of ${\bm{\epsilon}_l\rightarrow{\bf 0}}$ and for $k\geq 1$, $\hat{\balpha}_1,...,\hat{\balpha}_C$ that satisfies
		\begin{eqnarray} 
			\balpha_l
			&=& \left(\frac{\lambda{k}}{2}\,\diag\{|\balpha_l|^{\elementwise{(k-2)}}\} + {\bf X}^T{\bf X}\right)^{-1}{\bf X}^T{\bf y}_l,\ \ \ l = 1,...,C,
			\label{eqn_qnorm_alpha_solution_thm3}
		\end{eqnarray}
		minimizes 
		\begin{equation}\label{eqn_q_norm_obj888_thm3}
			({\bf y}_l-{\bf X}\balpha_l)^T({\bf y}_l-{\bf X} \balpha_l) + \lambda\kmeask{\balpha_l} ,
		\end{equation}
		for each target-parameter pair $\{ {\bf y}_l,\balpha_l\}$, $l = 1,...,C$ when the matrix $(\frac{\lambda{k}}{2}\,\diag\{|\balpha_l|^{\elementwise{(k-2)}}\} + {\bf X}^T{\bf X})$ is invertible. This happens for sure as soon as 
		\begin{equation}\label{eqn_nonsingular_reg_thm3}
			\frac{\lambda{k}}{2} \max_{j\in\{1,...,D\}} (|\alpha_{j,l}|^{(k-2)})   < \min_{j\in\{1,...,D\}} ({\rm eig}_j ({\bf X}^T{\bf X})).
		\end{equation}
\end{thm}
\begin{proof}
		Since the $l=1,...,C$ outputs are independent of each other, the regression for estimating each $\balpha_l$, $l=1,...,C$ can be performed independently. Hence the result.
\end{proof}
\begin{thm}\label{thm4}
		Consider under-determined systems ${\bf y}_l={\bf X}\bm{\alpha}_l$, $l=1,...,C$, where ${\bf
			y}_l\in\Real^M$ is the given target vector for each output, ${\bf X}\in\Real^{M\times D}$ is the regressor matrix and $\bm{\alpha}\in\Real^D$ is the parameter vector, with number of samples $M<D$ regressor dimensions.	
		Assume ${\bf X}{\bf X}^T$ and ${\bf X}\Xpower$ are of full rank for certain $k>1$. Then for that $k>1$ and under the limiting case of ${\bm{\epsilon}_l\rightarrow{\bf 0}}$, the stationary point given by
		\begin{eqnarray}
			\hat{\balpha}_l &=&
			\sgn\left(\btheta_l\right)\circ 
			\left|{\bf X}^T\left({\bf X}{\bf X}^T\right)^{-1}{\bf X}\btheta_l^{\elementwise{(k-1)}}\right|^{\elementwise{\frac{1}{k-1}}}  ,
			\label{eqn_qnorm_alpha_solution_dual_thm4}
		\end{eqnarray}
		where
		\begin{eqnarray}
			\btheta_l
			&=& \Xpower\left[{\bf X}\Xpower\right]^{-1} {\bf y}_l ,
			\label{eqn_qnorm_alpha_theta_thm4}
		\end{eqnarray}
		minimizes
		\begin{equation}\label{eqn_q_norm_obj111_thm4}
			\kmeask{\balpha_l}
			\ \,{\rm subject\ to\ \,} {\bf y}_l={\bf X}\balpha_l,\ \ \ l=1,...,C .
		\end{equation}
\end{thm}
\begin{proof}
		Since the $l=1,...,C$ outputs do not depend on each other, the regression for estimating each $\balpha_l$, $l=1,...,C$ can be performed independently. Hence the result.
\end{proof}

\subsection{Variance Analysis}

	In this subsection, we analyze the variance of the estimate and observe the essential properties.
	We shall work on the single output case only since the multiple outputs case is a direct stacking of the single output case. Assume the data is generated according to ${\bf y} = {\bf X}\balpha + \bm{\epsilon}$ with ${\bf X}\in\Real^{M\times D}$, $\balpha\in\Real^D$ where $\bm{\epsilon}$ is a zero mean noise with covariance matric ${\bf C}$. For the over-determined case, we have $M\geq D$. Suppose the estimation is initialized by $\balpha_0$, then the expectation of $\hat{\balpha}$ is
	\begin{eqnarray}
		E[\hat{\balpha}] 
		&=& E\left[\left(\frac{\lambda{k}}{2}\,\diag\{|\balpha_0|^{\elementwise{(k-2)}}\} + {\bf X}^T{\bf X}\right)^{-1}{\bf X}^T({\bf X}\balpha + \bm{\epsilon})\right] \nonumber \\
		&=& \left(\frac{\lambda{k}}{2}\,\diag\{|\balpha_0|^{\elementwise{(k-2)}}\} + {\bf X}^T{\bf X}\right)^{-1}{\bf X}^T{\bf X}\balpha\nonumber \\
		&\neq& \balpha, \ \ \forall k>1,\lambda>0.
	\end{eqnarray}
	This shows that the estimation is biased for the ranges of our working $k$ values and $\lambda$ values. For the special case when $\lambda=0$, we have an unbiased ols estimation since $E[\hat{\balpha}] = E\left[\left({\bf X}^T{\bf X}\right)^{-1}{\bf X}^T({\bf X}\balpha + \bm{\epsilon})\right] = E\left[\left({\bf X}^T{\bf X}\right)^{-1}{\bf X}^T {\bf X}\balpha \right] = \balpha$.

	For the under-determined system where $M<D$, the expectation of estimate is
	{\small
		\begin{eqnarray}
			E[\hat{\balpha}] 
			&=& E\left[\sgn\left(\btheta\right)\circ 
			\left|{\bf X}^T\left({\bf X}{\bf X}^T\right)^{-1}{\bf X}\btheta^{\elementwise{(k-1)}}\right|^{\elementwise{\frac{1}{k-1}}}\right],\
			\btheta
			=\Xpower\left[{\bf X}\Xpower\right]^{-1} ({\bf X}\balpha + \bm{\epsilon}) , \nonumber \\
			&\neq& \balpha.
		\end{eqnarray}
	}
	The inequality holds because the rank of
	\[ \left|{\bf X}^T\left({\bf X}{\bf X}^T\right)^{-1}{\bf X}\left\{\Xpower\left[{\bf X}\Xpower\right]^{-1} {\bf X}\balpha \right\}^{\elementwise{(k-1)}}\right|^{\elementwise{\frac{1}{k-1}}} \]
	is at most $M$ (due to $\left({\bf X}{\bf X}^T\right)^{-1}$) which is smaller than $D$. The above analyses for the over- and the under-determined cases show that both the estimates are biased, and this is consistent with the compressed estimation where some ideal parameters have been suppressed.

	For variance of the over-determined case, we have
	\begin{eqnarray}
		E[(\hat{\balpha} - E[\hat{\balpha}])( \hat{\balpha} - E[\hat{\balpha}])^T]
		&=& E\left\{\left[\left(\frac{\lambda{k}}{2}\,\diag\{|\balpha_0|^{\elementwise{(k-2)}}\} + {\bf X}^T{\bf X}\right)^{-1}{\bf X}^T\bm{\epsilon}\right] \right.\nonumber\\
		&& \hspace{5mm}\times\left.\left[\left(\frac{\lambda{k}}{2}\,\diag\{|\balpha_0|^{\elementwise{(k-2)}}\} + {\bf X}^T{\bf X}\right)^{-1}{\bf X}^T\bm{\epsilon}\right]^T \right\}
		\nonumber \\
		&=& \left[\left(\frac{\lambda{k}}{2}\,\diag\{|\balpha_0|^{\elementwise{(k-2)}}\} + {\bf X}^T{\bf X}\right)^{-1}{\bf X}^T\right] E[\bm{\epsilon}\bm{\epsilon}^T] \nonumber\\
		&& \hspace{5mm}\times\left[\left(\frac{\lambda{k}}{2}\,\diag\{|\balpha_0|^{\elementwise{(k-2)}}\} + {\bf X}^T{\bf X}\right)^{-1}{\bf X}^T\right]^T 
		\nonumber \\
		&=& \left[\left(\frac{\lambda{k}}{2}\,\diag\{|\balpha_0|^{\elementwise{(k-2)}}\} + {\bf X}^T{\bf X}\right)^{-1}{\bf X}^T\right] {\bf C} \nonumber\\
		&& \hspace{5mm}\times\left[\left(\frac{\lambda{k}}{2}\,\diag\{|\balpha_0|^{\elementwise{(k-2)}}\} + {\bf X}^T{\bf X}\right)^{-1}{\bf X}^T\right]^T .
	\end{eqnarray} 
	When $k=2$, it reduces to ridge regression where standard analysis applies. For other $k$ values, the situation becomes complicated.

	For the under-determined case, it is difficult to simplify $E[(\hat{\balpha} - E[\hat{\balpha}])( \hat{\balpha} - E[\hat{\balpha}])^T]$ due to the nonlinearity incurred by the absolute exponent.
	However, when $k=2$, it also reduces to the minimum norm solution case where standard analysis applies. Again, for other $k$ values, the situation becomes complicated.  


\subsection{Algorithm}\label{sec_algorithm} 

The algorithm for the proposed proximal bridge regression can be readily implemented in Python as shown in the function codes below. Here, both the over- and under-determined cases have been implemented according to the shape of the regression matrix. For numerical stability under practical considerations, all the inverse terms have included regularization. 
\\

		%
		%

\noindent{\bf Python function:}\\*[-6mm]
{\scriptsize
\begin{verbatim} 
import numpy as np
from numpy.linalg import inv 
		
def pbridge(X,Y,k,Lambda):
    # alpha: computed regression coefficient
    # X: data input matrix of size mxd, m is sample size, d is feature dimension
    # Y: output target vector/matrix of size mx1/mxq, q is output dimension
    # k: k-norm value in (1,2]
    # lambda: constrain coefficient in [0, positive integer] 
    if X.shape[0] < X.shape[1]: #under-determined
        if k==2:
            Xwarp = X.T
        else:
            Xwarp = np.absolute(X.T)**(1/(k-1))
        LamID =  Lambda * np.identity(X.shape[0])
        theta = np.array(Xwarp @ inv(X@Xwarp + LamID) @ Y, dtype=np.complex)
        alpha = theta**(k-1)
        alpha = X @ alpha
        alpha = inv(X@X.T + LamID) @ alpha
        alpha = X.T @ alpha
        alpha = (np.absolute(alpha))**(1/(k-1))
        alpha = np.sign(theta)*alpha
    else:  #over-determined
        Xcov = X.T@X
        pseudoinv = inv(Xcov + Lambda*np.identity(X.shape[1])) @ X.T
        alpha = pseudoinv @ Y
        if k<2:
            Lamfactor = Lambda*k*0.5
            for kk in range(0,4):
                A = np.array((np.absolute(alpha[:,0]+10**-10)**(k-2))*Lamfactor)
                pseudoinv = inv(Xcov + np.diag(A)) @ X.T
                alpha = pseudoinv @ Y
    return alpha.real
\end{verbatim}
}

\section{Case Studies} 

In this section, we perform some empirical studies under the Matlab platform on a real-world data set and several simulated benchmark data sets namely, the prostate cancer data, the Exclusive-OR (XOR) problem, and four simulated examples. Our goal here is to observe the behavior of the proposed method comparing with state-of-the-arts. For the prostate cancer data and the XOR problem, both the coefficient profiles and the estimation results will be observed. For the four simulated examples, the estimation results are obtained through Monte Carlo experiments to observe their statistical behaviors.


As a representative example for over-determined systems, the prostate cancer data from \cite{Stamey1}, which has been used by \cite{Hastie1,Tibshirani1,Zou1} as a benchmark example, is adopted to learn a linear regression model based on 67 training samples. This data set has a continuous response with 8 input variables. These inputs together with an intercept term give rise to 9 estimation coefficients for the linear regression model and this forms an over-determined regression system. 
On the other hand, as a representative example for under-determined systems, a 3rd-order polynomial model given by
\begin{equation} \label{eqn_poly_3rd_order}
	p(x_1,x_2) = \alpha_0 + \alpha_1x_1 + \alpha_2x_2 + \alpha_3x_1^2 + \alpha_4x_2^2 + \alpha_5x_1x_2 + \alpha_6x_1^3 + \alpha_7x_2^3 + \alpha_8x_1^2x_2 + \alpha_9x_1x_2^2,
\end{equation}
is deployed to learn the well-known XOR problem with four training data samples. The inputs to the XOR problem are $(x_1,x_2)\in\{(0,1),(2,1),(1,0),(1,2)\}$ with their corresponding target outputs given by $y\in\{0,0,1,1\}$, respectively. The system formed by learning the XOR data using the polynomial model constitutes an under-determined system since the number of parametric coefficients ($\alpha_0,...,\alpha_9$) is larger than the four learning samples. A total of 200 test samples for this XOR problem has been generated for mapping evaluation. These test samples have been generated by a bivariate Gaussian random number generator with centers located at the four training points, each center corresponds to 50 samples with an identity covariance matrix scaled by 0.3. The responses of these test data follow the labels of the four centers respectively.

\subsection{Over-determined system: prostate cancer example}

\noindent{\bf Coefficient Profile: } The profile of each coefficient estimate is observed with respect to variation of shrinkage settings following \cite{Hastie1,Tibshirani1,Zou1}. We shall first show the coefficient profile for the well-known ridge regression according to \cite{Hastie1} for immediate reference. Fig.~\ref{fig_profile_ridge_overdetermined} shows the ridge coefficient estimates plotted as function of $df(\lambda) = {\rm tr}[{\bf X}({\bf X}^T{\bf X}+\lambda{\bf I}){\bf X}^T]$, the \textit{effective degrees of freedom} implied by the penalty $\lambda$ (see \cite{Hastie1}, Section 3.4).

\begin{figure}[hhh]
	\begin{center}
		\epsfxsize=5.8cm
		\epsffile[7     4   630   523]{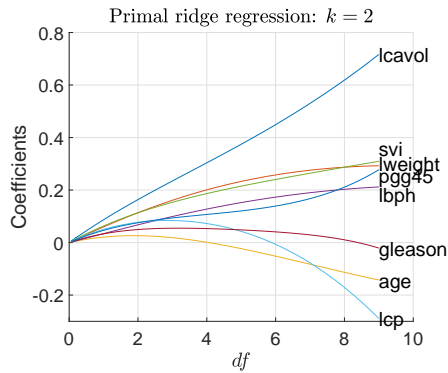}
		\caption{Profiles of ridge coefficients for the over-determined system: prostate cancer example.}
		\label{fig_profile_ridge_overdetermined}
	\end{center}
\end{figure}

As one of our interests here is to check the compression capability of the primal $p$-bridge regression in Theorem~\ref{thm1}, we shall observe the coefficient estimation at $k=1$ and compare it with that of the well-known the least absolute shrinkage and selection operator (lasso) \cite{Tibshirani1} (see \cite{Matlab} for the library codes in Matlab). 
Fig.~\ref{fig_profile_bridge_lasso_overdetermined} shows the coefficient profile of primal $p$-bridge regression at $k=1$ and that of lasso. This plot shows much resemblance of shrinkage behaviors between $p$-bridge and lasso for several coefficients, particularly for that of `lcp', `age' and `gleason' which shrunk to zero with a similar order of sequence when lowering the $df$ (raising the $\lambda$) value. However, for $p$-bridge regression, the coefficients of `pgg45', `lweight', `svi', and `lcavol' do not appear to reach zero sequentially as that of lasso. Moreover, for $p$-bridge regression, the coefficient of `pgg45' does not terminate together with that of `lbph' as in the case of lasso.

\begin{figure}[hhh]
	\begin{center}
		\begin{tabular}{cc}
			\epsfxsize=5.8cm
			\epsfysize=5.0cm
			\epsffile[7     4   629   522]{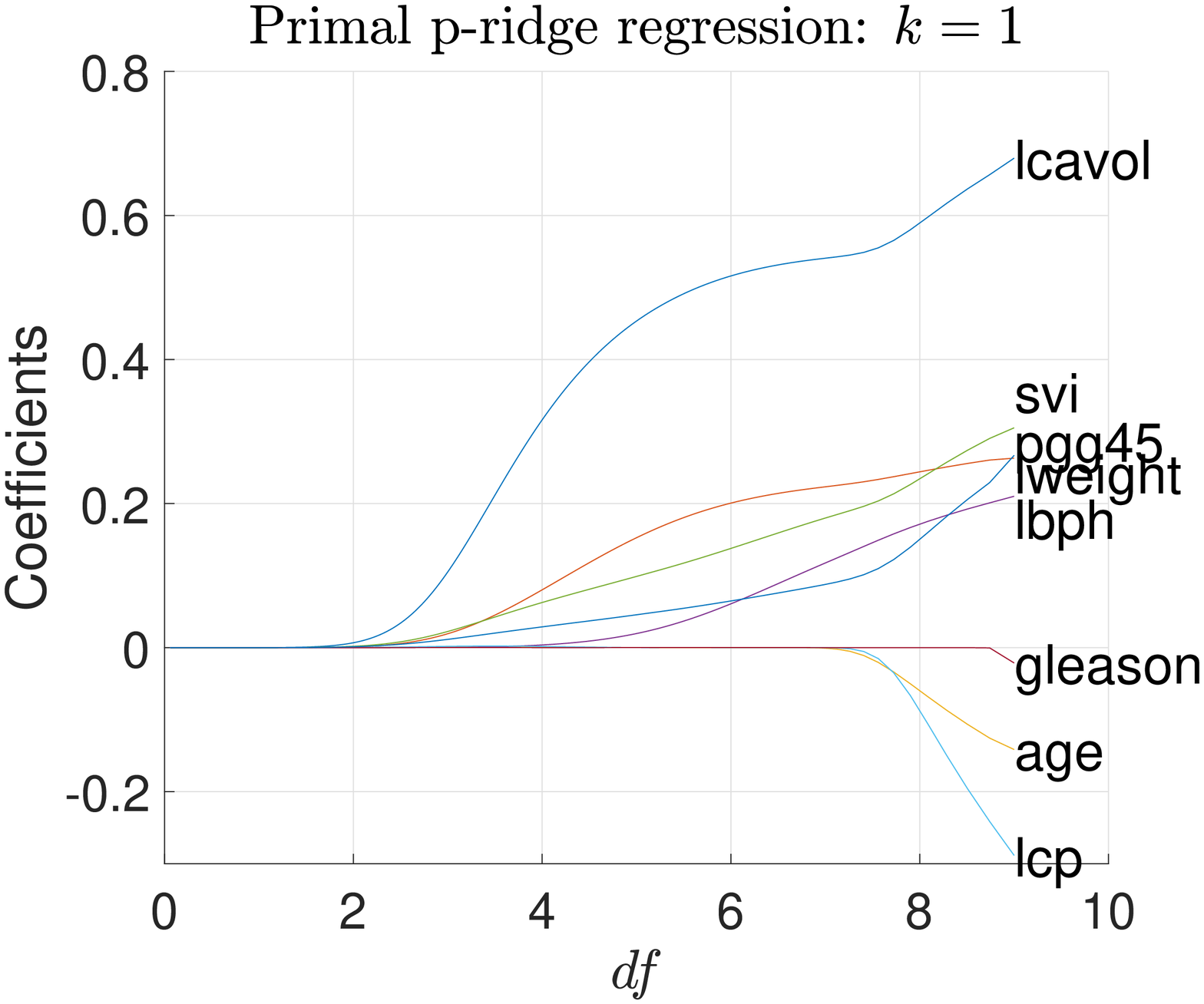}& \hspace{8mm}
			\epsfxsize=5.8cm
			\epsfysize=5.0cm
			\epsffile[56    24   641   470]{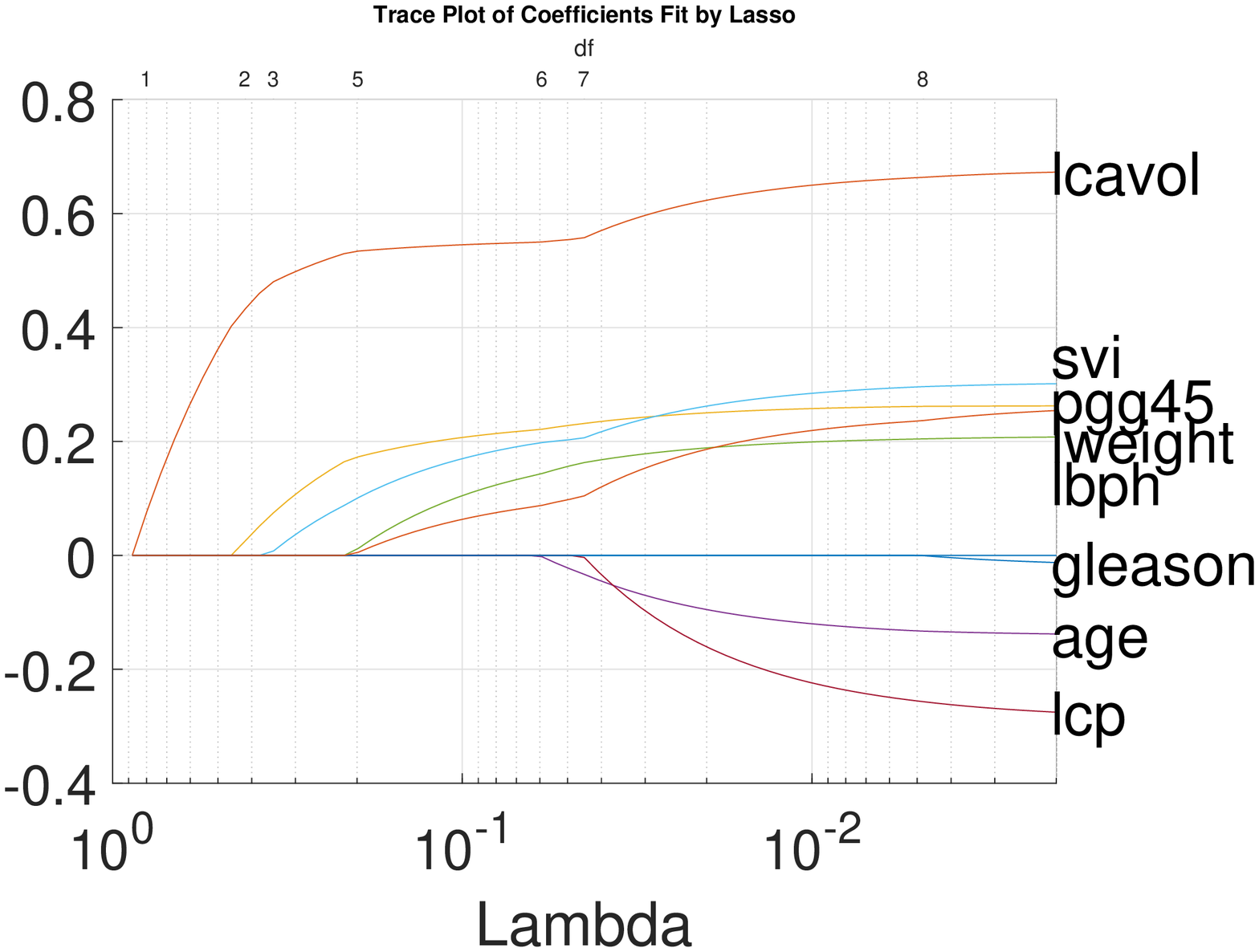}
			\\ \hspace{1cm} (a) primal $p$-bridge at $k=1$ & \hspace{2cm} (b) lasso
		\end{tabular}
		\caption{Profiles of primal $p$-bridge (at $k=1$) and lasso coefficients for the over-determined system: prostate cancer example.}
		\label{fig_profile_bridge_lasso_overdetermined}
	\end{center}
\end{figure}

Different from the lasso which is hinged upon the $\ell_1$-norm by varying only the `\texttt{lambda}' value of the Matlab library function, the $p$-bridge regression has an additional degree of freedom to vary the coefficient shrinkage (i.e., both $\lambda$ and $k$ are adjustable). 
Fig.~\ref{fig_profile_var_k_L1_L50_overdetermined} shows the variation of coefficient estimates over $k$ values at $\lambda=5$ and at $\lambda=50$. These plots show a high impact of $k$ values on coefficient shrinkage when $\lambda$ is large.\\
\begin{figure}[hhh]
	\begin{center}
		\begin{tabular}{cc}
			\epsfxsize=5.8cm
			\epsffile[7    10   642   522]{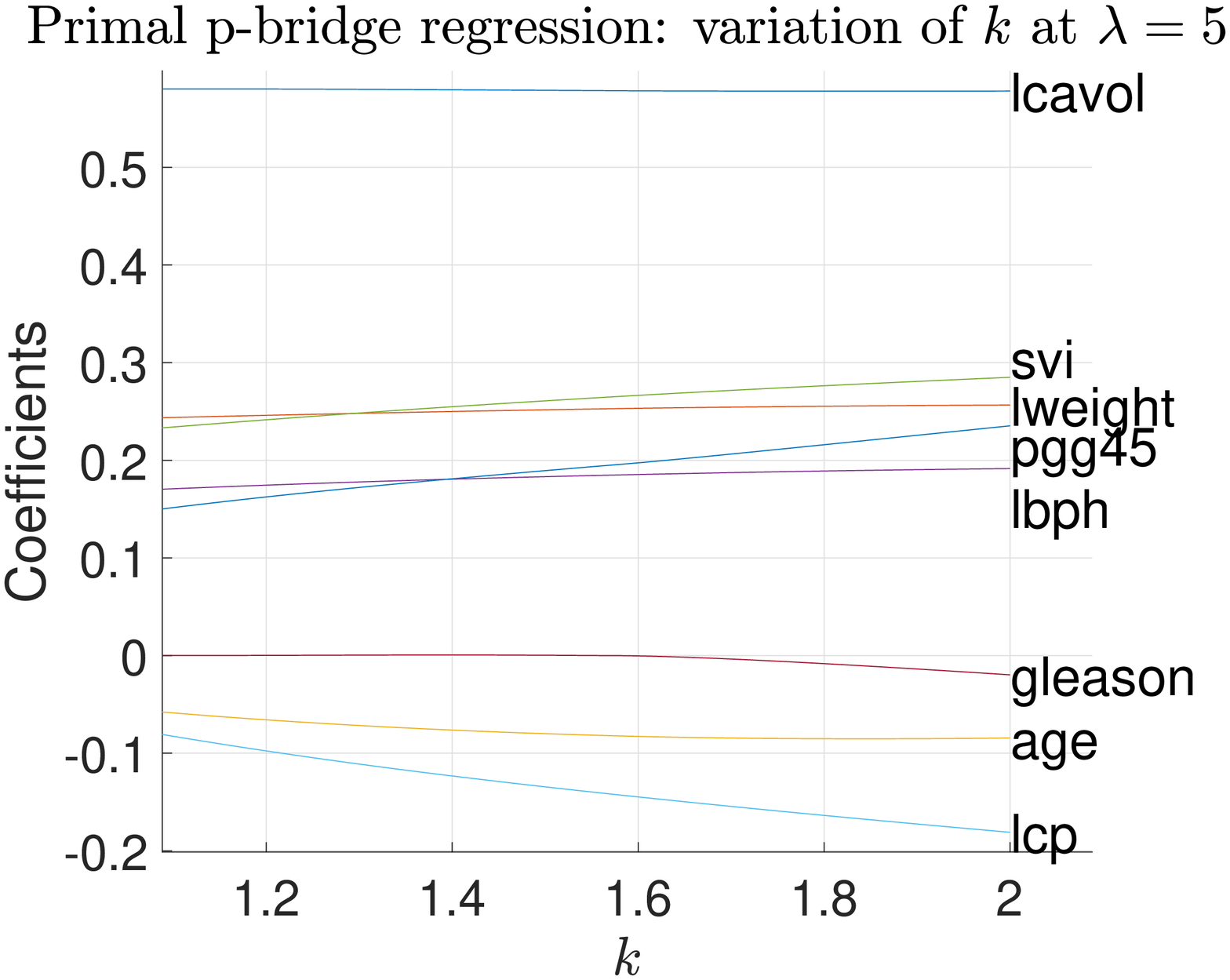}&
			\hspace{8mm}
			\epsfxsize=5.8cm
			\epsffile[7    10   642   522]{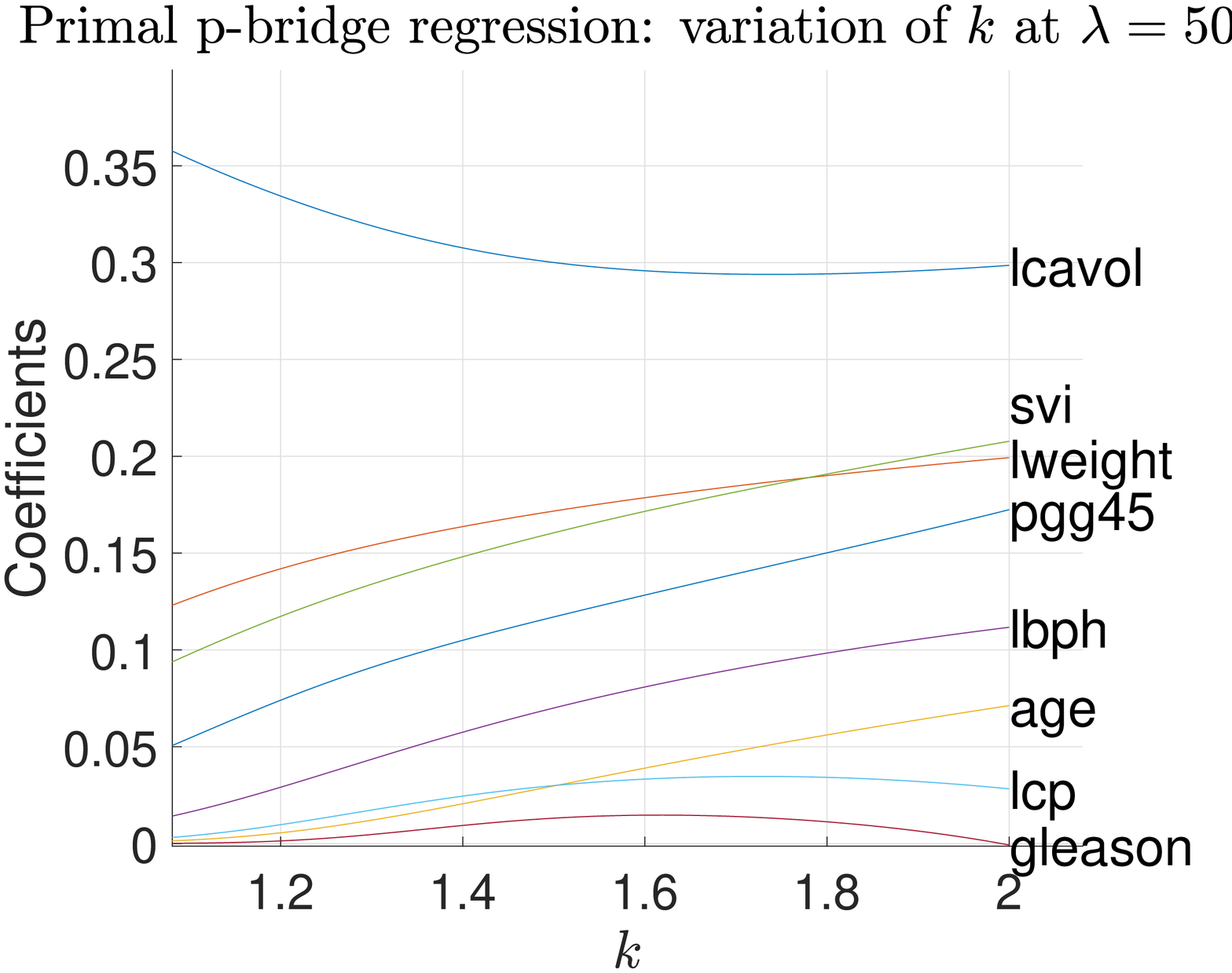}
			\\ \hspace{1cm} (a) $\lambda=5$ & \hspace{2cm} (b) $\lambda=50$
		\end{tabular}
		\caption{Profiles of primal $p$-bridge coefficients (by variation of $k$ values at $\lambda=5$ and $\lambda=50$) for the over-determined system: prostate cancer example.}
		\label{fig_profile_var_k_L1_L50_overdetermined}
	\end{center}
\end{figure}

\noindent{\bf Prediction results: } The prediction results of $p$-bridge is compared with several state-of-art methods namely, the ordinary least squares regression (ols) \cite{Duda1,Hastie1}, the ridge regression (ridge) \cite{Duda1,Hastie1}, the lasso \cite{Tibshirani1}, and the elastic-net \cite{Zou1}. Table~\ref{table_prostate_MSE} shows the results of the best chosen models obtained from tenfold cross-validation based on the 67 training observations. The mean-squared errors (MSE) between the estimated output and the measured output of the compared methods are reported based on the 30 test samples. These results show comparable performance of the primal $p$-bridge with those of the well-known state-of-the-arts. The estimated coefficients are shown in Table~\ref{table_prostate_coeff} for each method. Here we note that the chosen models are not sparse in favor of the cross-validated MSE.

\begin{table}[hhh]
	\caption{Prostate example: test mean squared error (MSE) and variables (excluding the intercept) selected}
	\label{table_prostate_MSE}
	\begin{center}\scriptsize
		\begin{tabular}{|c|c|c|c|}\hline
			Method  		   				& Tuned Parameter(s) 	  					  	& Test MSE 			& Variables selected \\ \hline
			ols     		   				& --  		  								  	& 0.520 (0.174) 	& All  \\ \hline
			ridge regression   				& $\lambda=1$								  	& 0.516 (0.175) 	& All \\ \hline
			lasso at ($\texttt{Alpha}=1$)	& $\texttt{Lambda}=0.02$    					& 0.483 (0.160)  	& (1,2,3,4,5,6,8) \\ \hline
			elastic-net        				& $\texttt{Lambda}=0.06$, $\texttt{Alpha}=0.11$	& 0.492 (0.164) 	& (1,2,3,4,5,6,8) \\ \hline
			$p$-bridge regression at ($k=1$)	& $\lambda=2$						  			& 0.494 (0.167)		& (1,2,3,4,5,6,8) \\ \hline
			$p$-bridge regression  				& $\lambda=2$, $k=1$						  	& 0.494 (0.167)  	& (1,2,3,4,5,6,8) \\ \hline
		\end{tabular}
	\end{center}
\end{table}

\begin{table}[hhh]
	\caption{Prostate example: estimated coefficients for the chosen setting based cross-validation on the training set}
	\label{table_prostate_coeff}
	\begin{center}\scriptsize
		\begin{tabular}{|l|r|r|r|r|r|r|}\hline
			Predictor	& ols		& ridge		& lasso 	& elastic-net 	& $p$-bridge at $k=1$   & $p$-bridge	\\ \hline
			0. intcpt   &  2.452	&  2.416 	&  2.467	&  2.466 		&  2.452			&  2.452 	\\ \hline
			1. lcavol	&  0.716	&  0.690 	&  0.624	&  0.590		&  0.637			&  0.637	\\ \hline
			2. lweight 	&  0.293	&  0.292	&  0.250 	&  0.254 		&  0.256 			&  0.256	\\ \hline
			3. age	    & -0.143	& -0.135	& -0.095	& -0.102		& -0.106 			& -0.106	\\ \hline
			4. lbph		&  0.212	&  0.210	&  0.189	&  0.197		&  0.193			&  0.193	\\ \hline
			5. svi		&  0.310	&  0.304	&  0.262	&  0.274		&  0.274			&  0.274 	\\ \hline
			6. lcp		& -0.289 	& -0.256	& -0.161	& -0.157		& -0.196			& -0.196 	\\ \hline
			7. gleason 	& -0.021	& -0.011	&  	0		&  	0			& -0.000			& -0.000	\\ \hline
			8. pgg45	&  0.277	&  0.258	&  0.186	&  0.199		&  0.206			&  0.206	\\ \hline
		\end{tabular}
	\end{center}
\end{table}


\subsection{Under-determined system: the XOR problem}

\noindent{\bf Coefficient Profile: }
The under-determined formulation in Theorem~\ref{thm2} ($p$-bridge regression in dual form) is applied to learn the XOR problem for coefficient profiling since there are more coefficients than training samples. 
Fig.~\ref{fig_profile_lasso_underdetermined} shows the coefficient profiles of dual $p$-bridge and lasso. This plot shows a highly sparse estimation for $p$-bridge at low $k$-value ($k=1.05$) comparing with lasso. Particularly, when the shrinkage penalty is low (at small $\lambda$ value), $p$-bridge suppresses all coefficients except those of $x_1^3$ and $x_2^3$ whereas lasso emphasizes the coefficients of $x_1$ and $x_2$ more than that of $x_1^3$ and $x_2^3$ while suppressing all other coefficients.
Fig.~\ref{fig_profile_bridge_var_k_underdetermined}(a) shows that for the range $1.05\leq k \leq 1.1$, $p$-bridge suppresses most of the coefficients except for the coefficients of $x_1^3$ and $x_2^3$ even at $\lambda=0$. The effect of sparseness begins to vanish after $k>1.1$ as seen from Fig.~\ref{fig_profile_bridge_var_k_underdetermined}(b).
The decision boundaries for both the dual $p$-bridge regression and the lasso in Fig.~\ref{fig_contour_bridge_lasso_underdetermined} show much resemblance in view of the high contribution of the coefficients of $x_1^3$ and $x_2^3$.  \\
\begin{figure}[hhh]
	\begin{center}\vspace{3mm}
		\begin{tabular}{cc}
			\epsfxsize=5.8cm
			\epsfysize=5.0cm
			\epsffile[32     4   806   520]{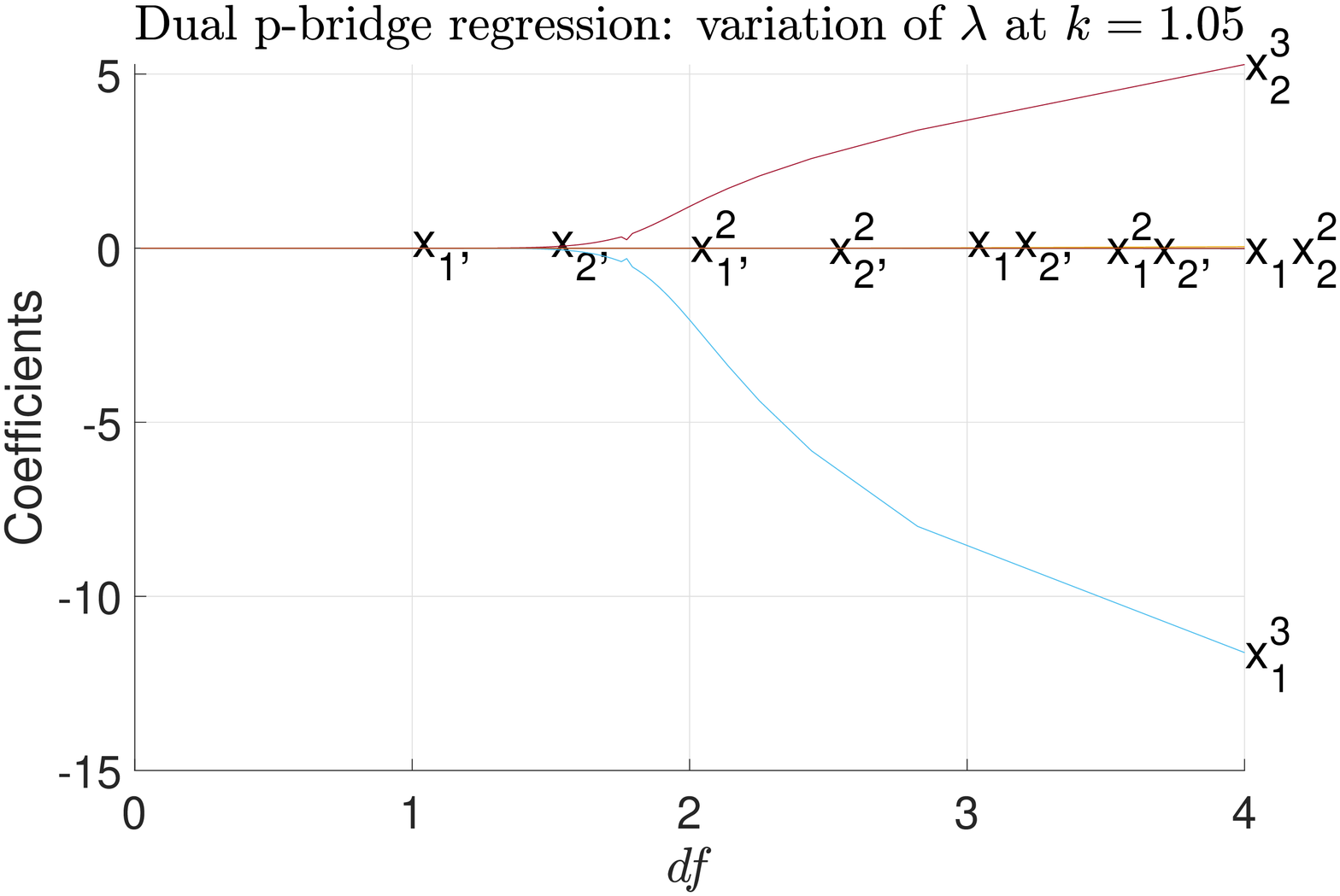}& \hspace{8mm}
			\epsfxsize=5.8cm
			\epsfysize=5.0cm
			\epsffile[88    24   733   453]{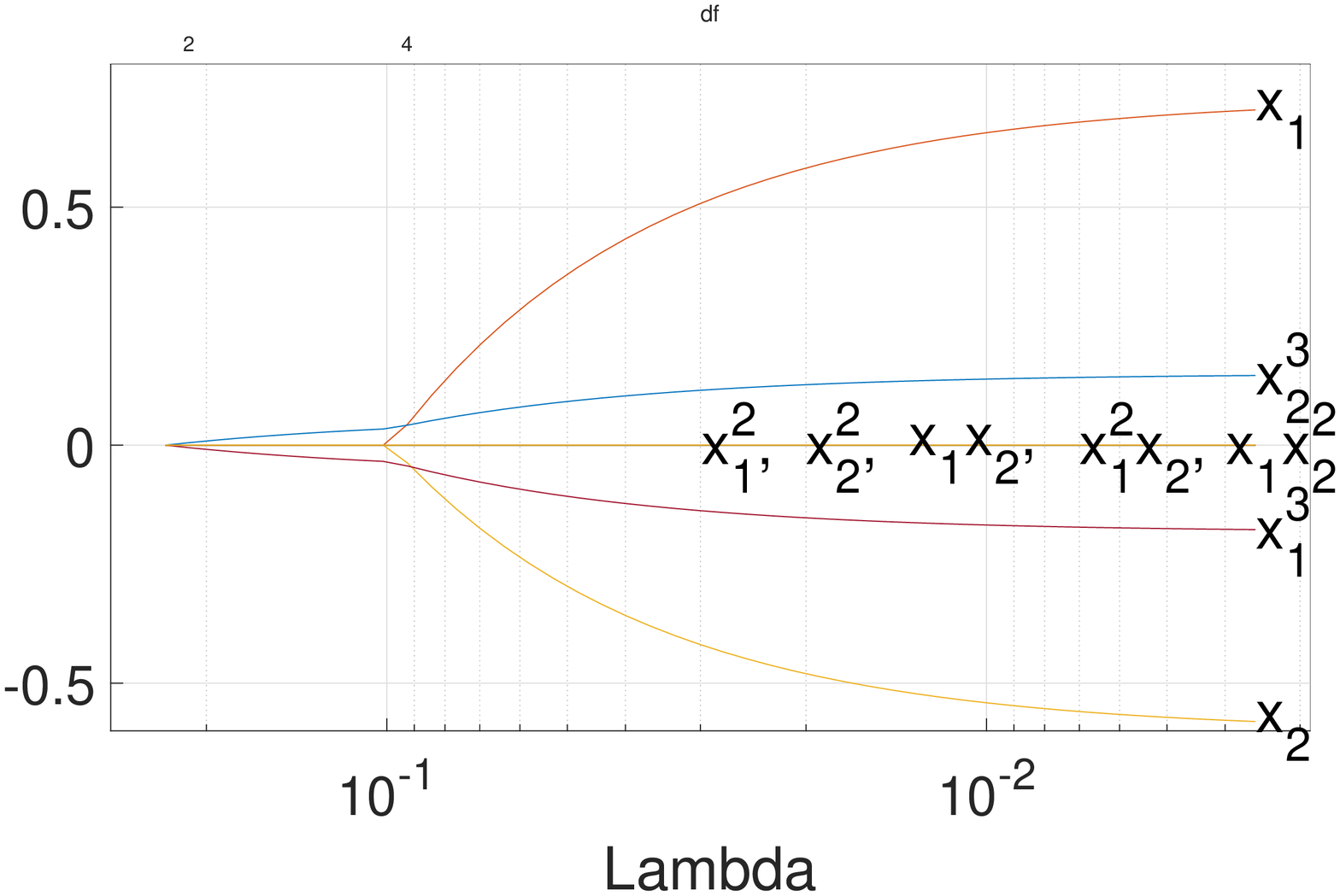}
			\\ \hspace{1cm} (a) dual $p$-bridge & \hspace{2cm} (b) lasso
		\end{tabular}
		\caption{Profiles of dual $p$-bridge and lasso coefficients for the under-determined XOR problem.}
		\label{fig_profile_lasso_underdetermined}
	\end{center}
\end{figure}

\begin{figure}[hhh]
	\begin{center}
		\begin{tabular}{cc}
			\epsfxsize=5.8cm
			\epsfysize=5.0cm
			\epsffile[32     9   816   522]{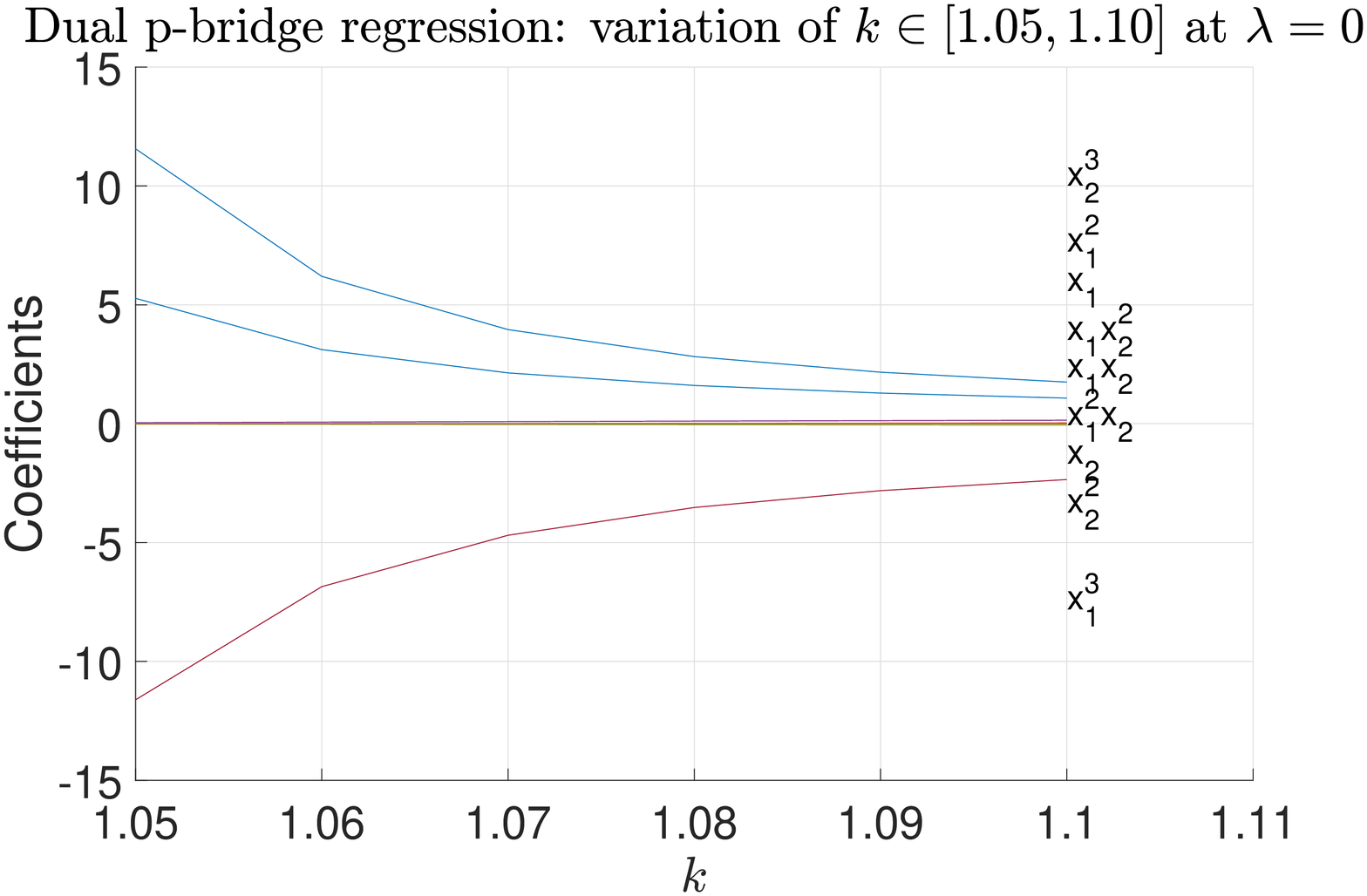}& \hspace{8mm}
			\epsfxsize=5.8cm
			\epsfysize=5.0cm
			\epsffile[37    10   810   526]{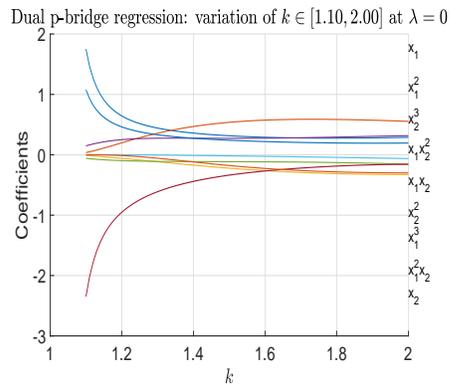}
			\\ \hspace{1cm} (a) & \hspace{2cm} (b)
		\end{tabular}
		\caption{Profiles of dual $p$-bridge coefficients (by variation of $k$ values at $\lambda=0$) for the under-determined XOR problem.}
		\label{fig_profile_bridge_var_k_underdetermined}
	\end{center}
\end{figure}

\begin{figure}[hhh]
	\begin{center}
		\begin{tabular}{cc}
			\epsfxsize=5.8cm
			\epsffile[50    26   481   372]{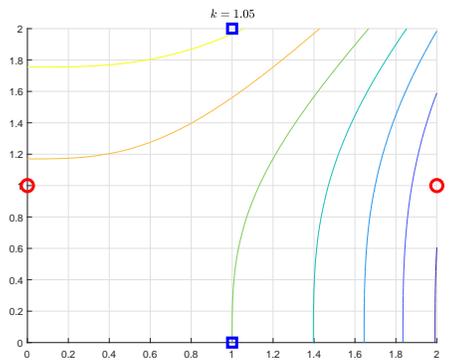}& \hspace{8mm}
			\epsfxsize=5.8cm
			\epsffile[36    19   388   298]{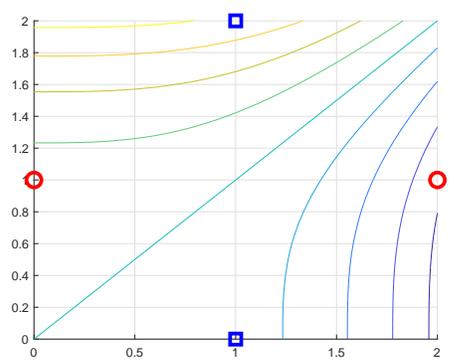}
			\\ \hspace{1cm} (a) dual $p$-bridge at $k=1.05$ & \hspace{2cm} (b) lasso
		\end{tabular}
		\caption{Decision contours of dual $p$-bridge and lasso.}
		\label{fig_contour_bridge_lasso_underdetermined}
	\end{center}
\end{figure}

\noindent{\bf Prediction results: }  Table~\ref{table_XOR_MSE} shows the results of chosen models based on training the four observations. The mean-squared errors between the estimated and the generated outputs of the compared methods are reported based on the 200 test observations. These results show comparable mapping performance of the dual $p$-bridge with that of the state-of-the-arts. The estimated coefficients as seen from Table~\ref{table_XOR_coeff} show sparseness for lasso at \texttt{Lambda} = 0.1 (\texttt{Alpha} = 1) and dual $p$-bridge at $k=1.05$ and $\lambda=30$.

\begin{table}[hhh]
	\caption{XOR: test mean squared error (MSE) and variables ($\alpha_0,...,\alpha_9$) selected based on the four training points}
	\label{table_XOR_MSE}
	\begin{center}\scriptsize
		\begin{tabular}{|c|c|c|c|}\hline
			Method  		   			  	& Tuned Parameter(s) 						& Test MSE 			& Variables selected 	\\ \hline	
			ols     		   			  	& --  		  								& 0.513 (0.089) 	& All 					\\ \hline
			ridge regression   			  	& $\lambda=6$								& 0.503 (0.047) 	& (0,1,2,3,4,6,7,8,9) 	\\ \hline
			lasso at ($\texttt{Alpha}=1$ and $\texttt{Lambda}=0.1$) 	& --			& 0.225 (0.011)  	& (0,6,7) 				\\ \hline
			lasso at ($\texttt{Alpha}=1$) 	& $\texttt{Lambda}=0$   					& 0.799 (0.189)  	& (0,1,2,3,4,6,7,8,9) 	\\ \hline
			elastic-net        			  	& $\texttt{Lambda}=0$, $\texttt{Alpha}=0.01$& 0.799 (0.189) 	& (0,1,2,3,4,6,7,8,9) 	\\ \hline
			$p$-bridge regression at ($k=1.05$) & $\lambda=30$ 								& 0.504 (0.040)  	& (6,7) 				\\ \hline
			$p$-bridge regression  			  	& $\lambda=0$, $k=2$ 						& 0.513 (0.089)  	& All 					\\ \hline
		\end{tabular}
	\end{center}
\end{table}

\begin{table}[hhh]
	\caption{XOR: estimated coefficients for the chosen setting based on the four training points (parameters within parenthesis have been prefixed, parameters without parenthesis have been determined based on training MSE)}
	\label{table_XOR_coeff}
	\begin{center}\scriptsize
		\begin{tabular}{|l|r|r|r|r|r|r|}\hline
			\multicolumn{1}{|c|}{} & \multicolumn{4}{c|}{State-of-arts} & \multicolumn{2}{c|}{Proposed $p$-bridge regression} \\ \hline
			Predictor	  & ols		& ridge		& lasso ($\texttt{A}1$, $\texttt{L}0.1$) 	& elastic-net 	& ($k=1.05$) $\lambda=30$ & $k=2$, $\lambda=0$ \\ \hline
			0.  intcpt    &  0.288	&  0.200 	&  0.500	&  0.644 		&  0.000	&  0.288 \\ \hline
			1.  $x_1$     &  0.554	&  0.040 	&  0		&  0.002		&  0.000	&  0.554 \\ \hline
			2.  $x_2$ 	  & -0.329	& -0.040 	&  0		& -0.002		& -0.000	& -0.329 \\ \hline
			3.  $x_1^2$   &  0.316	& -0.038	&  0 	 	&  0.760 		&  0.000	&  0.316 \\ \hline
			4.  $x_2^2$   & -0.154	&  0.038	&  0 		& -0.914		& -0.000	& -0.154 \\ \hline
			5.  $x_1x_2$  & -0.063	& -0.000	&  0 		&  0.000		&  0.000	& -0.063 \\ \hline
			6.  $x_1^3$	  & -0.159	& -0.071	& -0.034	&  0.406		& -0.050	& -0.159 \\ \hline
			7.  $x_2^3$	  &  0.195 	&  0.071	&  0.034	&  0.272		&  0.054	&  0.195 \\ \hline
			8.  $x_1^2x_2$& -0.301	& -0.051	&  0	 	& -0.179		& -0.000	& -0.301 \\ \hline
			9. $x_1x_2^2$&  0.111	&  0.051	&  0		&  0.460		&  0.000	&  0.111 \\ \hline
		\end{tabular}
	\end{center}
\end{table}

\clearpage

\subsection{Simulation examples}

The simulated data are generated based on the true model $y = \balpha^T\bm{x} + \sigma\epsilon$ with $\epsilon\sim N(0,1)$
according to several over-determined settings \cite{Tibshirani1,Zou1} below.
\begin{itemize}
	\item Example 1: 50 data sets of 20/20/200 (which denotes respectively the number of observations in the training/validation/test sets) with $\balpha=[3, 1.5, 0, 0, 2, 0, 0, 0]^T$ and $\sigma=3$. The correlation between $x_i$ and $x_j$ has been set at $0.5^{|i-j|}$.
	\item Example 2: similar to Example 1 except with $\balpha=[0.85, 0.85, 0.85, 0.85, 0.85, 0.85, 0.85, 0.85]^T$.
	\item Example 3: 50 data sets consisting of 100/100/400 samples with 
	\begin{eqnarray}
		\balpha &=& [\underbrace{0,...,0},\underbrace{2,...,2},\underbrace{0,...,0},\underbrace{2,...,2}]^T \nonumber\\ && \hspace{5mm}10\hspace{8mm}10\hspace{8mm}10\hspace{8mm}10 \nonumber
	\end{eqnarray}
	and $\sigma=15$. The correlation between $x_i$ and $x_j$ has been set at $0.5$.
	\item Example 4: 50 data sets consisting of 50/50/400 samples with 
	\begin{eqnarray}
		\balpha &=& [\underbrace{3,...,3},\underbrace{0,...,0}]^T \nonumber\\ && \hspace{5mm}15\hspace{8mm}25 \nonumber
	\end{eqnarray}
	and $\sigma=15$. The predictors ${\bf X}$ has been generated based on
	\begin{eqnarray}
		{\bf x}_i &=& Z_1 + \varepsilon^x_i,\ \ Z_1\sim N(0,1),\ \ i=1,...,5,	\nonumber\\
		{\bf x}_i &=& Z_2 + \varepsilon^x_i,\ \ Z_2\sim N(0,1),\ \ i=6,...,10,	\nonumber\\
		{\bf x}_i &=& Z_3 + \varepsilon^x_i,\ \ Z_3\sim N(0,1),\ \ i=11,...,15,	\nonumber\\
		{\bf x}_i &\sim& N(0,1),\ \ {\bf x}_i\ \textrm{independent\ identically\ distributed},\ \ i=16,...,40.	
	\end{eqnarray}
	where $\varepsilon^x_i$ are i.i.d.'s with $N(0,0.01)$, $i=1,...,15$. 
\end{itemize}

Since the ground truth model parameters are known, the Mean Square Error (MSE) of the estimation is measured based on (see e.g., \cite{FuWJ1}):
\begin{equation} \label{eqn_MSE}
	{\rm MSE} = (\hat{\balpha} - \balpha)^T\frac{{\bf X}^T_t{\bf X}_t}{n}(\hat{\balpha} - \balpha),
\end{equation}
where ${\bf X}_t$ and $n$ are respectively the matrix containing the test samples and the test sample size.

Table~\ref{table_simulated_examples} shows the median of MSEs computed using \eqref{eqn_MSE} based on the 50 trials of simulation. These results have been obtained based on a search of parameters utilizing the training set. The search ranges for elastic-net are \texttt{Lambda} $\in$ [0:0.01:1, 2:1:10, 20:10:100, 200:100:1000, 2000:1000:10000] and \texttt{Alpha} $\in$ [0.01:0.01:1]. The lasso uses the same search range with the elastic-net for \texttt{Lambda} but at \texttt{Alpha} = 1. For $p$-bridge, the search ranges are $\lambda\in$ [0:0.01:1, 2:1:10, 20:10:100, 200:100:1000, 2000:1000:10000] and $k\in$ [1:0.01:2].
Table~\ref{table_simulated_examples_non_zero_coeff} shows the median number of non-zero coefficients for each compared method. These results show the sparseness of $p$-bridge regression at $k=1$ is lower than that of lasso and approaching that of elastic-net in many cases.
These simulations verify the effectiveness of sparse estimation under the over-determined scenario.

\begin{table}[hhht]
	\caption{Median mean-squared errors for the simulated examples based on 50 trials}
	\label{table_simulated_examples}
	\begin{center}\scriptsize
		\begin{tabular}{|c|c|c|c|c|}\hline
			Method  		   & Example 1 		& Example 2 		& Example 3 	  & Example 4 			\\ \hline
			ols     		   & 5.599 (0.573)	& 5.476 (0.531)	 	& 150.293 (8.131) & 1085.833 (112.153) 	\\ \hline
			ridge regression   & 3.864 (0.262)	& 1.819 (0.185)	 	&  22.540 (1.062) & 61.964 (4.254)	   	\\ \hline
			lasso			   & 2.982 (0.328)	& 3.515 (0.504)		&  45.950 (1.585) & 62.121 (4.753)    	\\ \hline
			elastic-net        & 3.020 (0.335)	& 1.842 (0.225)		&  23.999 (0.962) & 62.339 (3.933) 		\\ \hline
			$p$-bridge regression at $k=1$ & 2.756 (0.252) 	& 3.113 (0.456)		&  36.002 (1.387) & 47.543 (4.595) 		\\ \hline
			$p$-bridge regression  		   & 2.761 (0.232) 	& 1.817 (0.175)		&  24.043 (0.976) & 49.162 (4.136) 		\\ \hline
		\end{tabular}
	\end{center}
\end{table}

\begin{table}[hhht]
	\caption{Median number of non-zero coefficients}
	\label{table_simulated_examples_non_zero_coeff}
	\begin{center}\scriptsize
		\begin{tabular}{|c|c|c|c|c|}\hline
			Method  		   			 & Example 1& Example 2 & Example 3 & Example 4 \\ \hline
			ols     		   			 & 	8		& 	8 		&  40  		&  40		\\ \hline
			ridge regression   			 & 	8		& 	8 		&  40		&  40  		\\ \hline
			lasso			   			 & 	5		& 	6		&  20  		&  10 		\\ \hline
			elastic-net       			 & 	6		& 	8		&  40 		&  23		\\ \hline
			$p$-bridge regression at $k=1$ &  7		& 	8		&  36		&  33		\\ \hline
			$p$-bridge regression  		 &  8		& 	8		&  40		&  35		\\ \hline
		\end{tabular}\\
	\end{center}
\end{table}


\section{Experiments}

\subsection{NIPS Data sets}

In this section, we conduct experiments on the NIPS 2003 challenge data sets \cite{NIPS2003} to observe the estimation and prediction behaviors of the proposed $p$-bridge solution on data sets of relatively high dimension. These data sets consist of both under- and over-determined scenarios for binary classification. According to \cite{NIPS2003}, the task of the Arcene data set is to distinguish between cancer and normal patterns based on continuous input mass-spectrometric data. The task of Dexter is to filter texts about ``corporate acquisitions'' based on sparse continuous input variables. The task of Dorothea is to predict which compounds bind to Thrombin based on sparse binary input variables. The task of Gisette is to discriminate between two confusable handwritten digits, namely the digit four and the digit nine, based on sparse continuous input variables. The task of Madelon is to classify random data based on sparse binary input variables. Table~\ref{NIPS_data_sets} summarizes these data sets in terms of their data dimensions and sample sizes for training, validation and testing. 

\begin{table}[hhh]
	\caption{NIPS Feature Selection Challenge Data Sets}\label{NIPS_data_sets}
	\begin{center}
		{\scriptsize
			\begin{tabular}{|c|c|c|r|r|r|r|}
				\hline
				Dataset & Domain & Type & \# Feat. & \# Train & \# Valid. & \# Test \\
				\hline
				Arcene    & Mass spec.        & Dense     &  10000 &  100 &  100 &  700 \\
				Dexter    & Text categ.       & Sparse    &  20000 &  300 &  300 & 2000 \\
				Dorothea  & Drug discov.      & Sp. bin.  & 100000 &  800 &  350 &  800 \\
				Gisette   & Digit recog.      & Dense     &   5000 & 6000 & 1000 & 6500 \\
				Madelon   & Artifical data    & Dense     &    500 & 2000 &  600 & 1800 \\
				\hline
		\end{tabular}}
	\end{center}
\end{table}

\subsection*{(i) Comparison setup and protocol}

Similar to that in the section of case study, the state-of-the-art methods included in this experimental study are the ordinary least squares regression (ols, \cite{Duda1,Hastie1}), the ridge regression (ridge, \cite{Duda1,Hastie1}), the lasso, \cite{Tibshirani1}, and the elastic-net (\cite{Zou1}). The platform for this evaluation has been based on Python 3.9.7 running on an Intel i7 CPU of 2.8GHz with 16GB of RAM. 
In view of the stability in handling both the over- and the under-determined systems, the ols has been implemented based on numpy's pseudoinverse function (\texttt{numpy.linalg.pinv}) for computation of the weight coefficients estimate (i.e., using $\balpha=$ \texttt{pinv(${\bf X}$)}$@{\bf y}$). The ridge regression has utilized \texttt{sklearn.linear\_model.Ridge} function. 
The elastic-net has been implemented using \texttt{sklearn.linear\_model.MultiTaskElasticNet}. 
In this function, the lasso corresponds to setting \texttt{l1\_ratio} $= 1$ (where \texttt{l1\_ratio} is the mixing value that controls the relative balance between $\ell_2$and $\ell_1$ penalties), and the elastic-net corresponds to setting $0 <$ \texttt{l1\_ratio} $< 1$. The parameter \texttt{alpha} in \texttt{MultiTaskElasticNet} controls the overall strength of the penalty term which is composed of the $\ell_1$ and $\ell_2$ penalties mixture. Hence, for lasso the only tuning parameter is the strength of constraint \texttt{alpha} $\geq 0$ (with \texttt{l1\_ratio} fixed at 1) while for elastic-net the tuning parameters are \texttt{alpha} $\geq 0$ and \texttt{l1\_ratio} $\in(0,1)$. Finally, parallel to lasso and elastic-net, we have included two versions of our proposed $p$-bridge in this study. The respective versions are $p$-bridge at $k=1.05$ which tunes only $\lambda$ and $p$-bridge which tunes both $\lambda$ and $k$ value. Here we note that there is fundamental difference between the $k$ value of $p$-bridge, which corresponds to the norm value itself, and \texttt{l1\_ratio} of lasso which mixes between $\ell_1$ and $\ell_2$ norms.

As the test labels of these data sets are not released to the public, we shall use the validation set to test the classification prediction. Except for ols, the hyper-parameters $\lambda$, $k$, \texttt{alpha} and \texttt{l1\_ratio} for the above methods have been determined based on a twofold cross-validation utilizing only the training set. These cross-validated hyper-parameters are subsequently utilized to retrain each method based on the entire training set for test prediction utilizing the unseen validation set. For tuning the hyper-parameters, the utilized search ranges are \texttt{l1\_ratio} $\in$ [0.01, 0.1:0.1:1], \texttt{alpha}, $\lambda\in$ [0:0.1:1, 2:1:10, 20:10:100, 200:100:1000] and $k\in$ [1:0.1:2]. 

\subsection*{(ii) Results and observation}

\noindent\textbf{Accuracy}: The results in terms of the classification prediction accuracy (which is the fraction of samples with their category being correctly predicted) are shown in Fig.~\ref{fig_estimated_Acc}. 
These results show that $p$-bridge with full tuning capability (i.e., with both $\lambda$ and $k$ adjustable) has good prediction accuracy relative to that of the state-of-arts in all data sets while $p$-bridge at $k=1.05$ shows inferior performance for the Arcene and the Dexter data sets. The lasso also shows under-performed predictions for these two data sets of under-determined systems even though these results are better than that of $p$-bridge at $k=1.05$. The elastic-net shows relatively under-performed prediction only for the Dexter data set. Attributed to the utilization of the stable sklearn library and the pseudoinverse implementation, the ridge and the ols are observed to have good prediction accuracy over all data sets. It is observed that for $p$-bridge, the chosen $k$-values for these five data sets based on training validation are respectively 1.7, 1.9, 1.5, 1.9 and 1.05. These results show that the Gaussian prior (at $k=2$) might not give the best fit in each case and $p$-bridge provides the alternatives. The additional degree of freedom provided by tuning the $k$ values on top of the penalty $\lambda$ term plays a part in determining a suitable model in face of data diversity.\\ 
\begin{figure}[hhh]
	\begin{center}
		\epsfxsize=12cm
		\epsfysize=6.0cm
		\epsffile[118    40  1193   547]{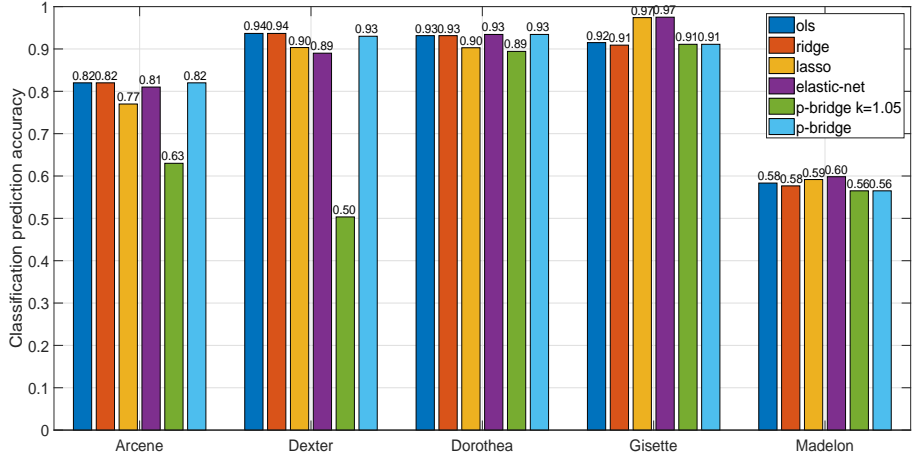}
		\caption{Classification test accuracies.)}
		\label{fig_estimated_Acc}
	\end{center}
\end{figure}

\noindent\textbf{Sparseness of estimated coefficients}: The estimated regression weight coefficients (sorted in ascending order) for each method on each data set are plotted in Fig.~\ref{fig_estimated_coefficients1}. These results show that the ols and the ridge have the densest estimation while the lasso and the elastic-net have the sparest estimation among the compared methods. The $p$-bridge at $k=1.05$ shows sparser estimation than that of the $p$-bridge with tuned $k$ and $\lambda$ values for all data sets. The $p$-bridge at $k=1.05$ show similar sparseness of estimation with lasso only for the Dexter data set, while lasso show the sparest estimation for the remaining data sets. \\

\noindent\textbf{Training processing time}: The training CPU processing time for each method is shown in Table~\ref{table_Training_CPU_times}. Due to the utilization of the computational intensive (but more stable) pseudoinverse in the ols, the fastest training CPU processing time goes to ridge as it has four data sets clocking the lowest processing time. Comparing $p$-bridge and elastic-net, it shows faster processing time in three data sets (Arcene, Dexter and Gisette) but slower processing time in two data sets (Dorthea and Madelon). The trend is similar for comparing between $p$-bridge at $k=1.05$ and lasso.
Here, we note that our implementation of the $p$-bridge has been based on the algorithm shown in section~\ref{sec_algorithm} without optimization of codes.

\begin{table}[hhht]
	\caption{Training CPU time in seconds}
	\label{table_Training_CPU_times}
	\begin{center}\scriptsize
		\begin{tabular}{|c|c|c|c|c|c|}\hline
			Method  			 & Arcene 	 & Dexter	 & Dorothea 	& Gisette 	& Madelon 	\\ \hline
			ols     			 & 	0.125	 &   1.000	 &  24.297	 & 157.578	  &  0.313	   \\ \hline
			ridge    		 	 &{\bf 0.063}&{\bf 0.109}&   3.391	 &{\bf 10.922}&{\bf 0.016} \\ \hline
			lasso				 & 	3.500	 &   1.250	 &{\bf 1.297}& 106.234	  &  0.063	   \\ \hline
			elastic-net   		 & 	2.266	 &   3.625	 &   2.531 	 & 108.594	  &  0.125	   \\ \hline
			$p$-bridge at $k=1.05$ &  0.125	 &   0.969	 &  14.984	 & 100.875	  &  0.313	   \\ \hline
			$p$-bridge		 	 &  0.172	 &   1.047	 &  16.547	 &  97.375	  &  0.313	   \\ \hline
		\end{tabular}\\
	\end{center}
\end{table}


Summarizing the experiments on the NIPS data sets, in terms of the prediction accuracy, the good performance relative to state-of-arts on all data sets shows the capability of the $p$-bridge to balance between the constraint weightage ($\lambda$) and the $k$-value which is related to the underlying $\ell_p$-norm prior. The sparseness of estimation for $p$-bridge is seen to be controlled by the $\lambda$ and $k$ values. For under-determined systems, the accuracy of prediction appears to be much affected by the $k$ values close to 1. In terms of training processing time, the current implementation of $p$-bridge shows faster processing speed than that of elastic-net in three data sets but slower processing speed in two data sets. These results show competing prediction accuracy and processing time with the state-of-the-arts.

\begin{figure}[bbbb]
	\begin{center}
		\begin{tabular}{cc}
			\epsfxsize=6.8cm
			\epsfysize=5.0cm
			\hspace{-2mm}\epsffile[33     6   544   420]{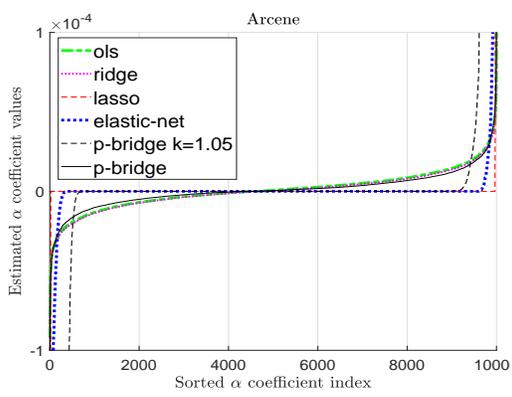} &
			\epsfxsize=6.8cm
			\epsfysize=5.0cm
			\hspace{-2mm}\epsffile[33     6   544   420]{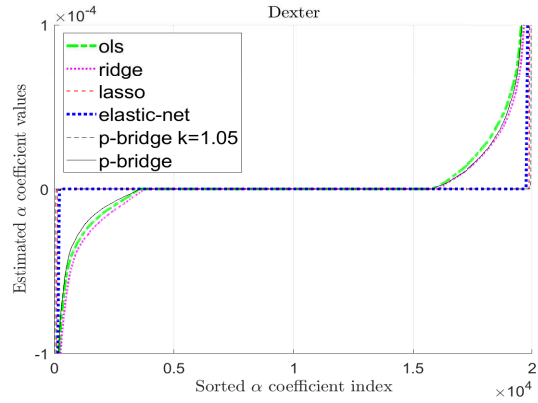}
			\\*[1mm]  \small{(a) Arcene}  & \small{(b) Dexter} \\*[2mm]
			\epsfxsize=6.8cm
			\epsfysize=5.0cm
			\hspace{-2mm}\epsffile[33     6   544   420]{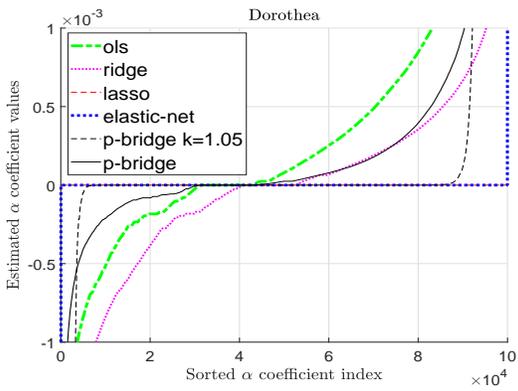} &
			\epsfxsize=6.8cm
			\epsfysize=5.0cm
			\hspace{0mm}\epsffile[33     6   544   420]{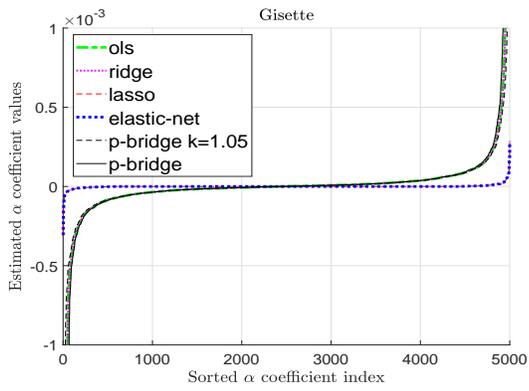}
			\\*[1mm]  \small{(c) Dorothea}  & \small{(d) Gisette}\\*[2mm]
		\end{tabular}
		\epsfxsize=6.8cm
		\epsfysize=5.0cm
		\hspace{-2mm}\epsffile[33     6   544   420]{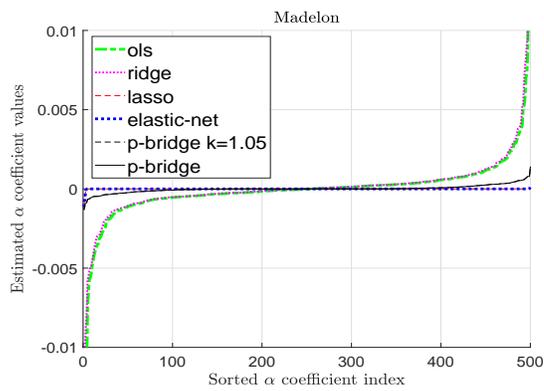} 
		\\*[1mm]  \small{(e) Madelon}
		\caption{Sorted weight coefficients for each data set}
		\label{fig_estimated_coefficients1}
	\end{center}
\end{figure}

\clearpage

\subsection{Recognition of Handwritten Digits} 

The goal of this experiment is to observe the stretching behavior of $p$-bridge for prediction between the $\ell_1$-norm and the $\ell_2$-norm minimization of the parameter vector, particularly for data sets with large regions of empty background as a form of multicollinearity.

The first database is for optical recognition of handwritten digits (abbreviated as Optdigit) where it was collected based on a total of 43 people \cite{kaynak1995methods,UCI1b}. The original 32$\times$32 bitmaps were divided into non-overlapping blocks of 4$\times$4 where the number of on pixels were counted within each block. This generated an input matrix of 8$\times$8 where each element was an integer within the range [0, 16]. The dimensionality is thus reduced from 32$\times$32 to 8$\times$8. Each of the 10 numerical digits constitutes a category for recognition. The total number of 5620 samples are divided equally into two sets for training and testing in our experiment. The left panel of Fig.~\ref{fig_Optdigit} shows some samples of the reduced resolution image taken from the training set (upper two rows) and the testing set (bottom two rows).
The second database is the MNIST data set of handwritten digits \cite{LeCun6,LeCun2} which is a popular benchmark for algorithmic study and experimental comparison. This data set contains a training set of 60,000 samples and a testing set of 10,000 samples where each sample image is of $28\times 28$ pixels resolution. Similar to the Optdigit, the MNIST data set has an output of 10 class labels. The right panel of Fig.~\ref{fig_Optdigit} shows some training (upper two rows) and testing (lower two rows) sample images from the MNIST data set. 


\begin{figure}[hhhh]
	\begin{center}
		\epsfxsize=6cm
		\epsffile[58    39   377   292]{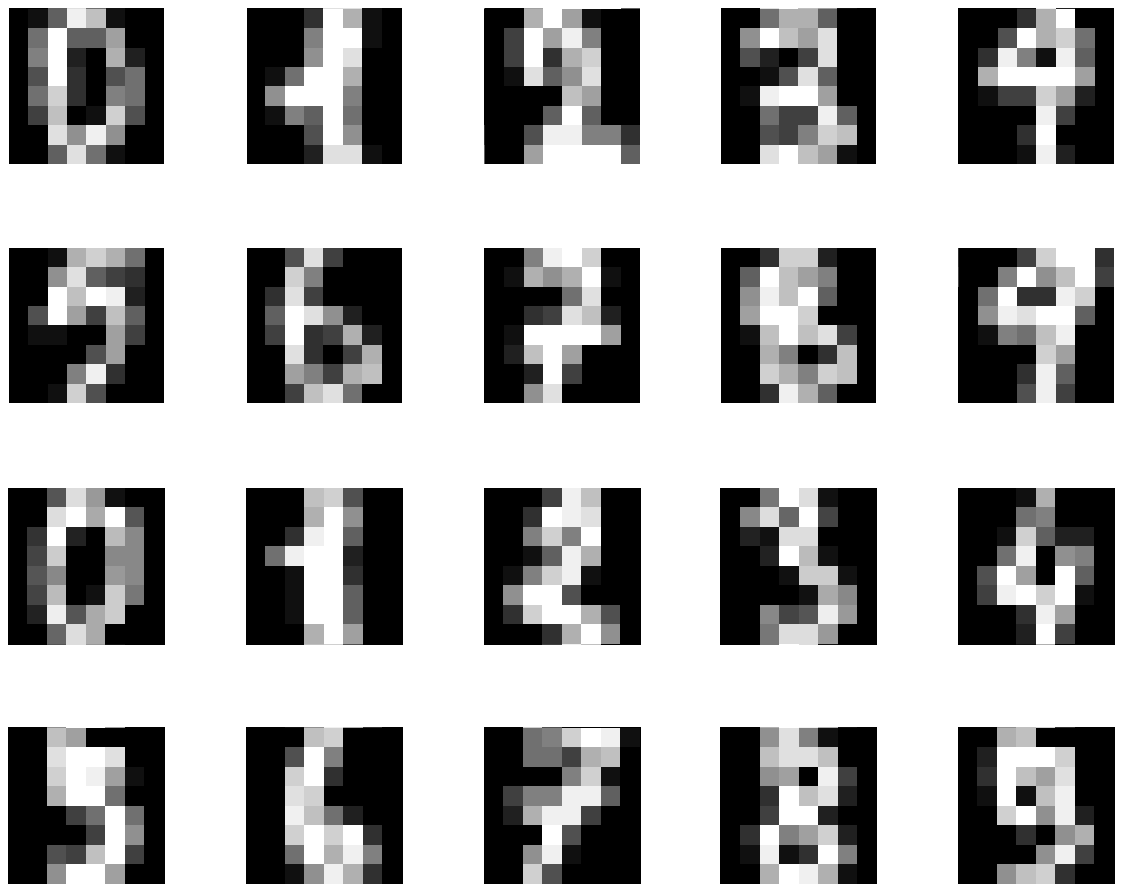} \hspace{1cm}	
		\epsfxsize=6cm
		\epsffile[55    34   380   292]{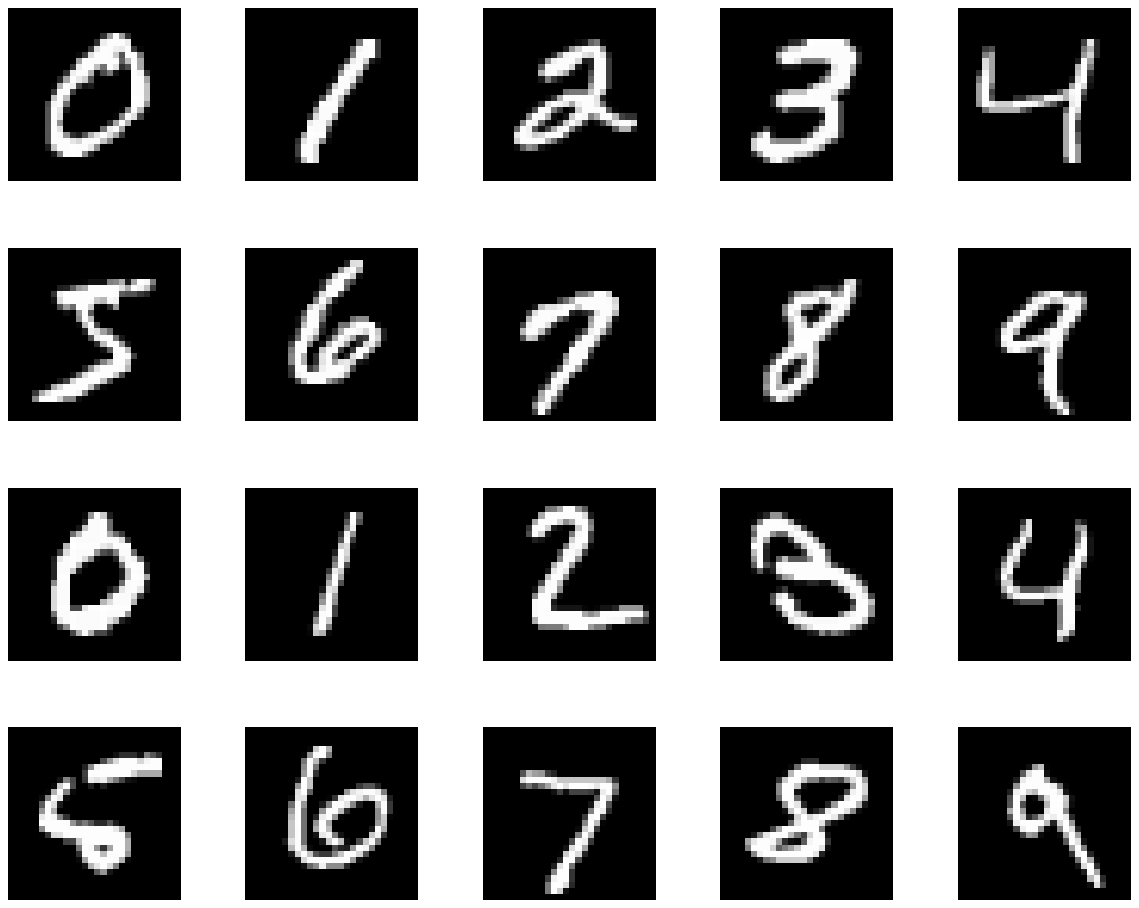} 
		\caption{Samples of Optdigit (left panel) and MNIST (right panel) data sets. Samples shown in the upper two rows are taken from the training set and samples shown in the bottom two rows are taken from the test set.}
		\label{fig_Optdigit}
	\end{center}
\end{figure}


\subsection*{(i) Experimental Setup}

The input images of the two data sets of handwritten digits are mapped to the polynomial space for discrimination beyond linear decision. For the Optdigit data set, each image (8$\times$8 pixels) is pooled at various sizes and reshaped into a row vector 
before expanded by a full polynomial of second order to generate the input features (1$\times$2145 including the bias term) for regression. This data set for Optdigit is divided equally with each of the training and test matrices of 2810$\times$2145 size. For the MNIST data set, a reduced polynomial \cite{Toh39} of third order has been applied to the pooled vector to generate the feature vector (1$\times$22661 including the bias term). The training matrix is of 60,000$\times$22661 size and the test matrix is of 10,000$\times $22661 size for MNIST. 

Both the over-determined and the under-determined settings of $p$-bridge will be evaluated.
For the over-determined setting, the entire training matrix is utilized for training. For the under-determined setting, only 10 samples (1 sample for each of the 10 digits) are utilized for training for both the databases. This is known as \textit{one-shot learning} in the community of computer vision. For both the over- and under-determined cases, the entire test set is utilized for evaluation.

\subsection*{(ii) Results and Observation}

{\bf Over-determined case:} Fig.~\ref{fig_acc_overdetermined} shows the test accuracies of the over-determined $p$-bridge for various $k$-values in $\{1.1,1.2,...,2.0\}$ for both the databases. Alongside, the test accuracies for \texttt{MultiTaskElasticNet} is plotted at different mixing values of \texttt{l1\_ratio} in $\{1.0,0.9...,0.3,0.2,0.001\}$ for Optdigit and MNIST, both at \texttt{alpha}=0.1. Here, \texttt{l1\_ratio}=1 indicates the lasso setting and \texttt{l1\_ratio}$<1$ indicates an elastic-net setting with \texttt{l1\_ratioa}$\rightarrow 0$ approaches the $\ell_2$-norm estimation. For both databases, the results show deviation from that of $\ell_2$-norm estimation when the $k$-values move away from 2.0 for $p$-bridge and when \texttt{l1\_ratio} values move away from 0. The deviation behaviors for $p$-bridge and elastic-net are apparently different. The $p$-bridge shows relatively stable prediction with a marginally up-trend whereas the elastic-net shows a deterioration trend when moving away from the $\ell_2$-norm estimation. The peak performance for $p$-bridge shows a 99.15\% accuracy at $k=1.3$ for Optdigit and a 97.01\% accuracy at $k=1.7, 1.9$ for MNIST. For elastic-net, the mixing of $\ell_2$-norm estimation with $\ell_1$-norm estimation appears to cause significant deterioration of prediction accuracy for these two data sets. 


\begin{figure}[hhh]
	\begin{center}
		\epsfxsize=12cm
		\epsffile[113    39  1304   672]{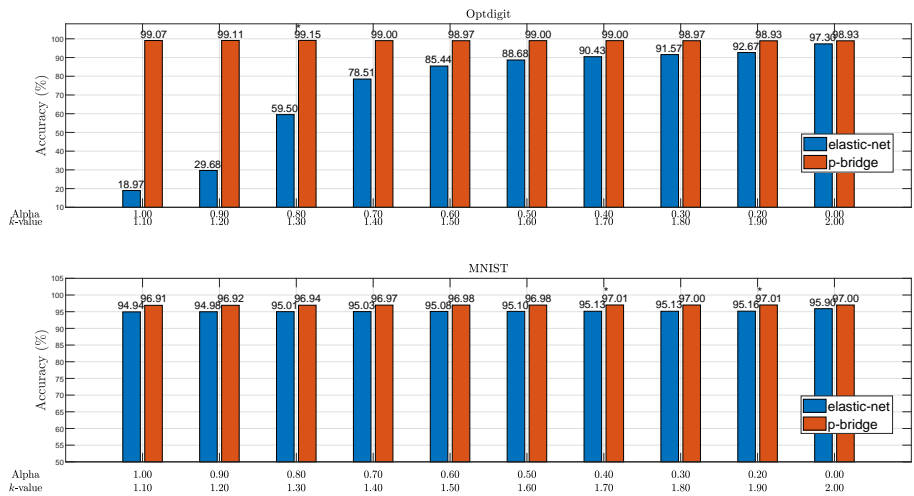} 	
		\caption{Accuracy results for the over-determined estimation. The asterisk (*) marks the highest accuracy achieved.}
		\label{fig_acc_overdetermined}
	\end{center}
\end{figure}

{\bf Under-determined case: } Fig.~\ref{fig_acc_underdetermined} shows the accuracies of prediction by the $p$-bridge and elastic-net under the under-determined (single-shot learning) scenario. These accuracies are plotted over variation of the $k$-values and the \texttt{l1\_ratio}-values respectively. Here, the prediction accuracy appears to degrade for both the $p$-bridge and the lasso for most cases when the estimations are moved away from the $\ell_2$-norm formulation ($k=2$ and \texttt{l1\_ratio} $\rightarrow 0$). However, the prediction accuracy of $p$-bridge peaks at $k=1.90$ and surpasses that based on the $\ell_2$-norm for both the databases.

\begin{figure}[hhh]
	\begin{center}
		\epsfxsize=12cm
		\epsffile[113    39  1304   672]{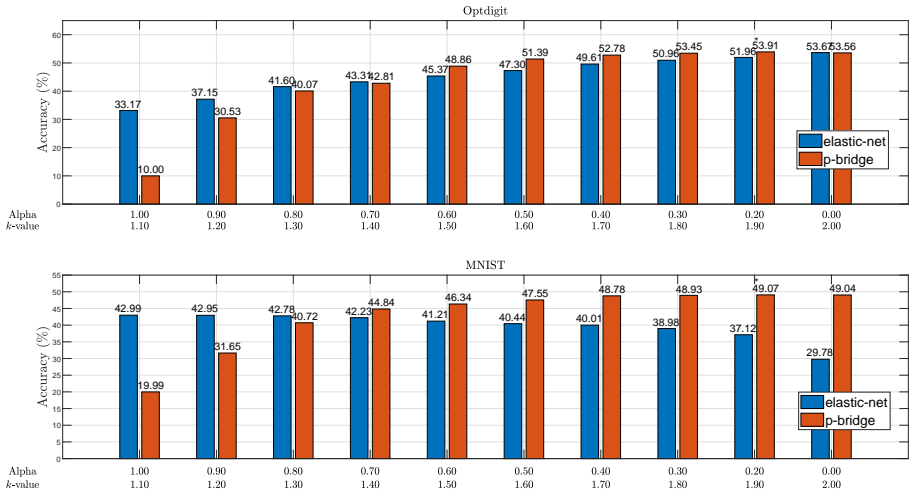} 	
		\caption{Accuracy results for the under-determined estimation (one-shot learning). The asterisk (*) marks the highest accuracy achieved.}
		\label{fig_acc_underdetermined}
	\end{center}
\end{figure}

{\bf Estimated coefficients: } Here, we study the estimated coefficients of $p$-bridge and compare them with those of lasso using the MNIST database under the over-determined setting. In order to observe the importance of each pixel in the estimation process, a weighting coefficient is tied to each pixel for the estimation. This gives rise to a linear regression model with $28\times 28+1$ parameters and 10 sets of weight coefficients corresponding to the one-hot encoded target digits. 
The lasso is set at (\texttt{l1\_ratio}=1, \texttt{alpha}=0.01) and the $p$-bridge is set at ($k=1$, $\lambda=10$). Fig.~\ref{fig_heatmap_mnist} shows the heatmap of the learned coefficient values for each of the 10 digits. In general, these heatmaps show sparseness of the coefficients due to the $\ell_1$ penalized learning for both lasso and $p$-bridge. In terms of the emphasized pixels with high absolute coefficient values, both lasso and $p$-bridge show similar locations, for both the positively emphasized (darker box) and the negatively emphasized (brighter box) pixels with variations. However, the value ranges of the coefficients differ for lasso and $p$-bridge. The lasso shows a lower prediction accuracy (82.5\%) than that of the $p$-bridge (86.1\%).

\begin{figure}[hhh]
	\begin{center}
		\epsfxsize=15cm
		\epsfysize=16cm
		\epsffile[188    73  1363   717]{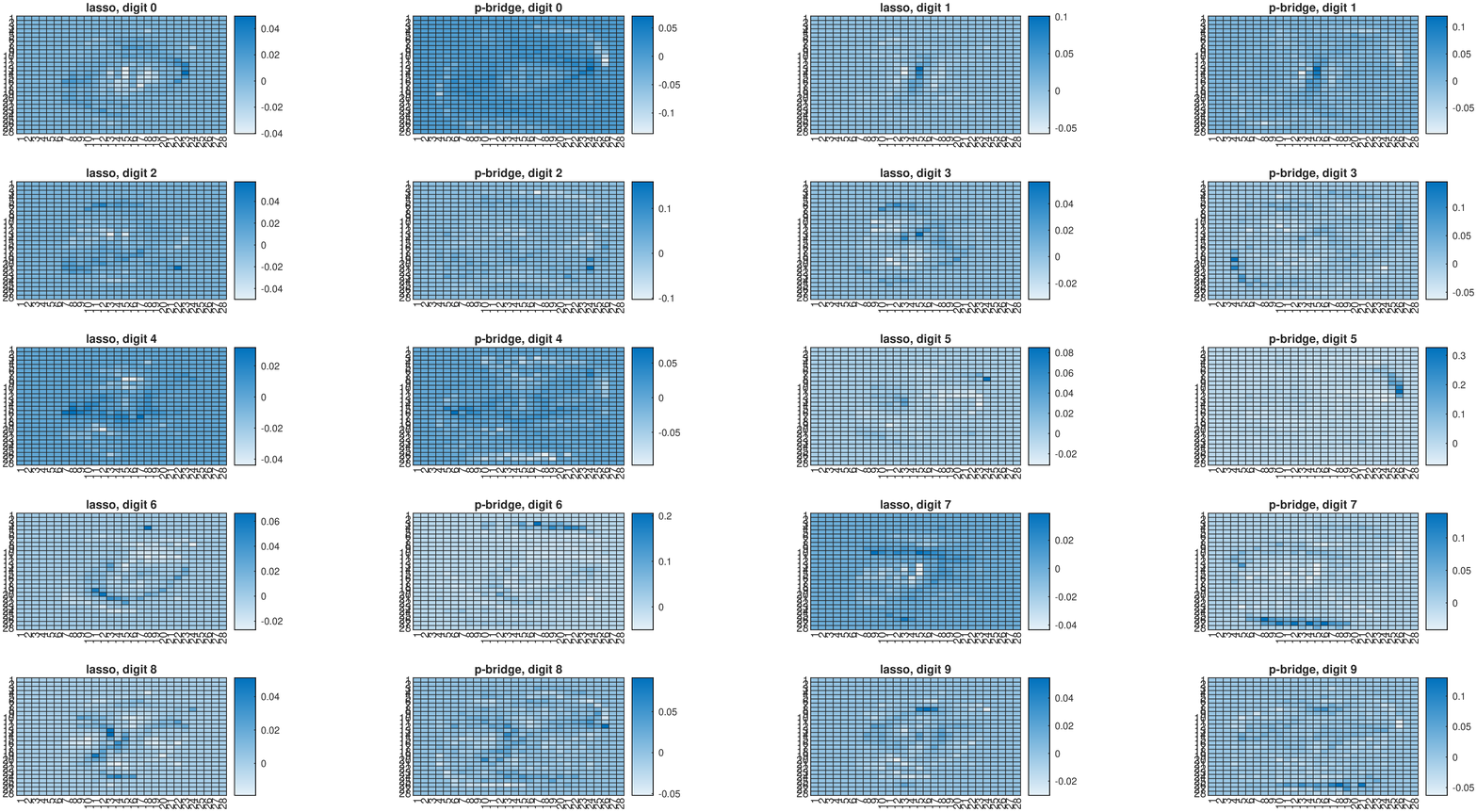} 	
		\caption{The estimated coefficients for each digit.}
		\label{fig_heatmap_mnist}
	\end{center}
\end{figure}

In order to observe which are the informative pixels responsible for discrimination, we plot the sum of training images and its logarithmic form plus the logarithm of the sum of absolute estimated coefficients corresponding to each of the pixels in Fig.~\ref{fig_totalsum_coefficients}. The main reason to plot in logarithmic form is to reveal near zero values which are not visible in the original plot. These images show close correspondence between the logarithm of sum of training images (the plot at top right) and that of the estimated coefficients (the plot at bottom right) of $p$-bridge. This results shows all the informative pixels of the image have been utilized by $p$-bridge. As for lasso, not all informative pixels have been utilized due to its crisp variable selection mechanism.

\setcounter{figure}{13}
\begin{figure}[hhh]
	\begin{center}
		\epsfxsize=12cm
		\epsffile[134    44  1007   508]{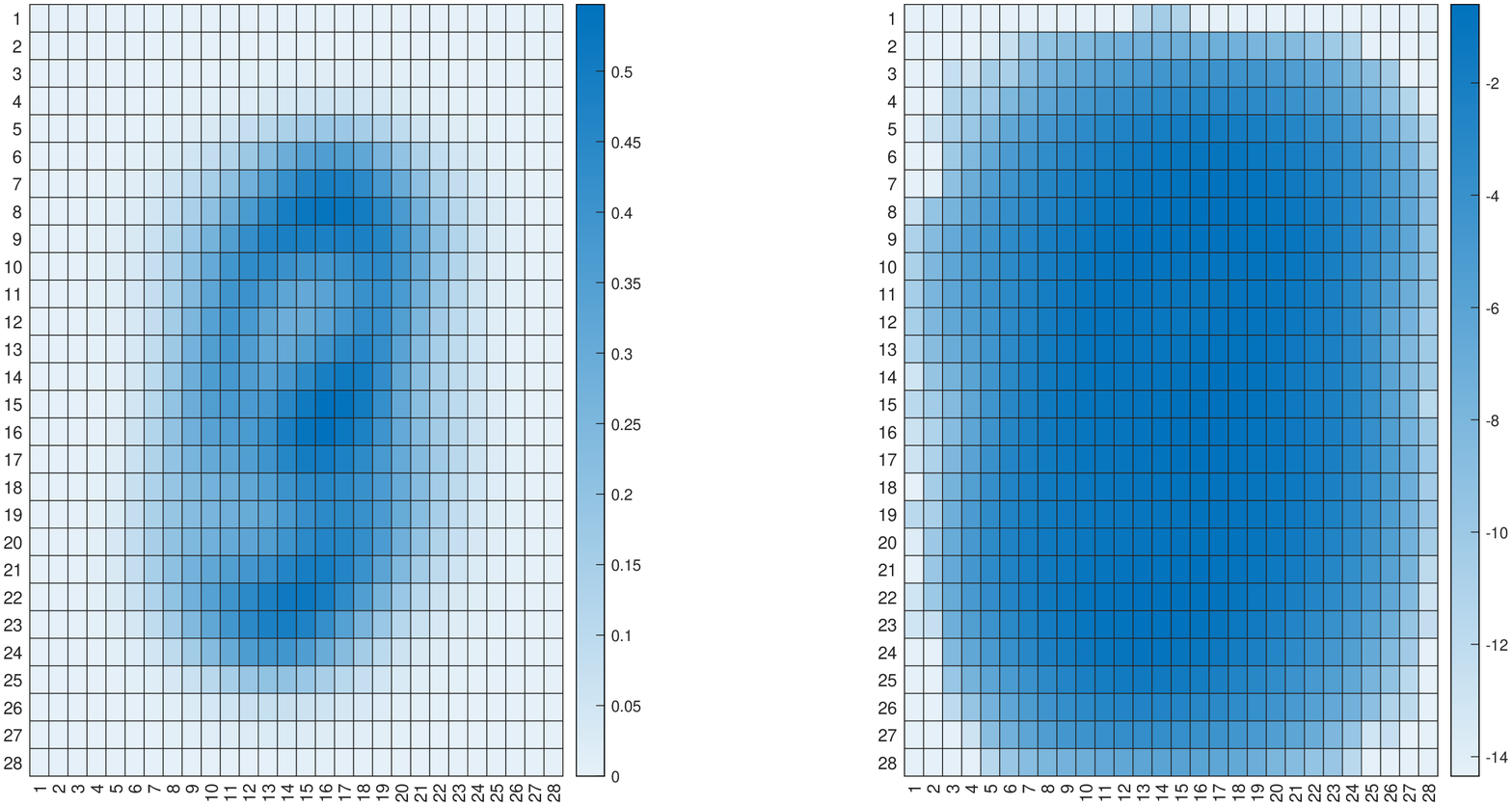} 	
		\caption{Heatmap of the sum of absolute coefficients.}
		\epsfxsize=12cm
		\epsffile[134    44  1007   508]{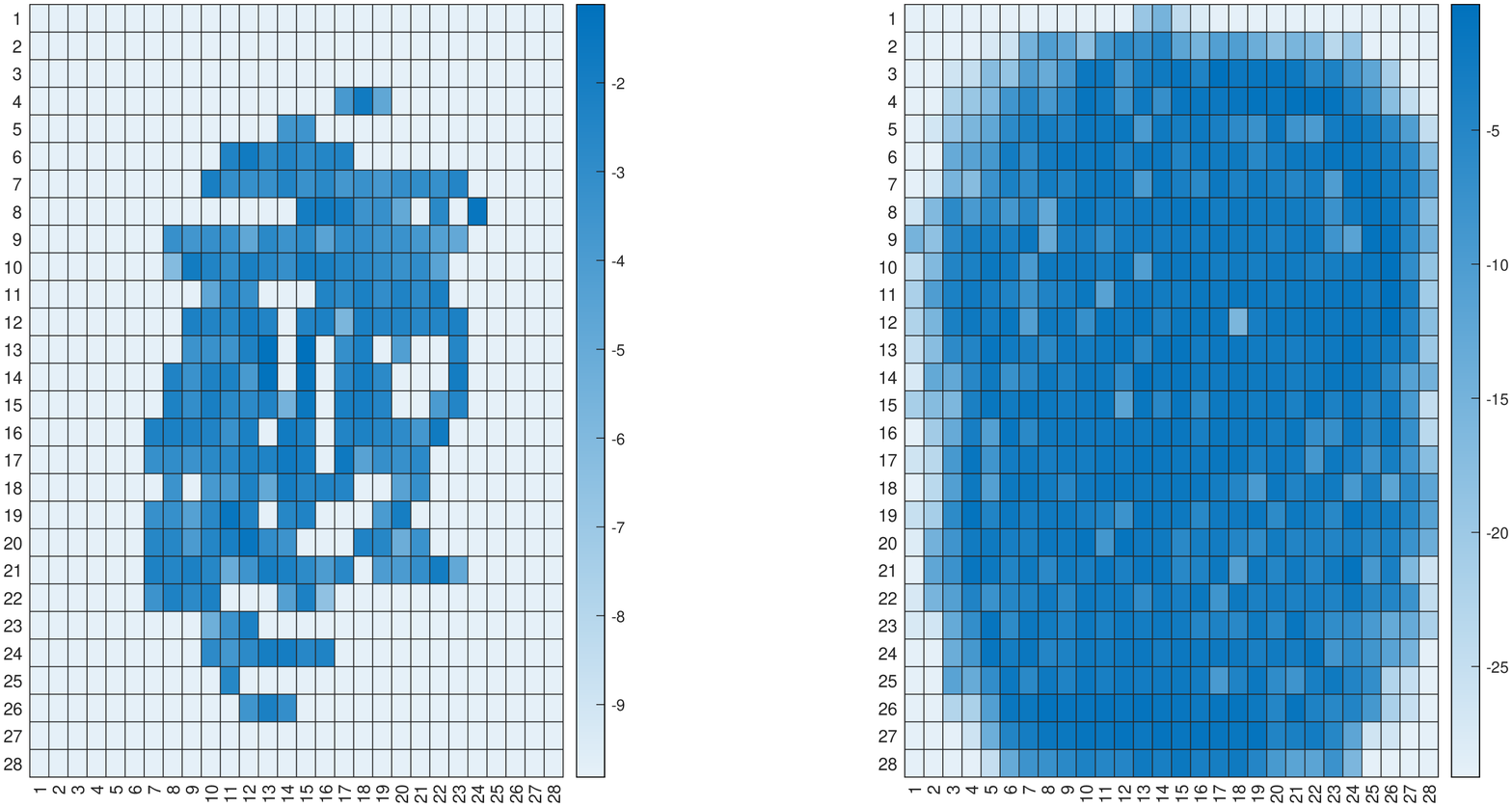} 	
		\caption{MNIST: The top two panels are heatmaps of the sum of training images (left) and its logarithmic (right) images. The bottom two panels are the heatmaps of the logarithm of sum of absolute coefficients for lasso (left) and $p$-bridge (right).}
		\label{fig_totalsum_coefficients}
	\end{center}
\end{figure}

Summarizing the experiments for Optdigit and MNIST, both the over- and the under-determined cases show peaking of prediction accuracy beyond $k=2$ ($\ell_2$-norm regression) for $p$-bridge. This reveals the importance to move away from the Gaussian prior.


\subsection{Summary of Results and Observations}

\noindent{\bf NIPS data sets: }  In this set of experiments on binary classification, we have observed the estimation and prediction behaviors of $p$-bridge on real-world and artificial data sets of large dimension at under-determined (Arcene, Dexter, and Dorothea) and over-determined (Gisette and Madelon) settings. The results and observations are summarized as follows:
\begin{itemize}
	\item Prediction accuracy: under various data situations, the $p$-bridge with adjustable $\lambda$ (strength of constraints) and $k$-value (related to the $\ell_p$-norm value) shows consistent and competing prediction accuracy relative to state-of-the-art methods namely, ols, ridge, lasso and elastic-net. For $p$-bridge at $k=1.05$ (high compressive setting), the prediction accuracy shows a compromised performance for some of the under-determined cases due to the heavy bias introduced with many regressors being suppressed.
	\item Sparseness of coefficients: in general, the $p$-bridge shows comparable sparseness of estimated coefficients relative to that of elastic-net but lower sparseness comparing to that of lasso. For under-determined systems, the $p$-bridge at $k=1.05$ achieves comparable sparseness of estimated coefficients by trading off the prediction accuracy.
	\item Training processing time: while the implemented $p$-bridge in Python was not code optimized, its training processing time shows competing processing speed with that of sklearn library's \texttt{MultiTaskElasticNet} except for the Dorothea and Madelon data sets.
\end{itemize}

\noindent{\bf Recognition of handwritten digits: }  In this set of experiments on multi-category classification based on multiple outputs formulation, the behavior of $p$-bridge is summarized as follows:
\begin{itemize}
	\item Prediction accuracy: for the over-determined setting, the $p$-bridge shows relatively stable prediction over various $k$ values with peak accuracy away from that at $k=2$. This shows the inferiority of the Gaussian prior in these cases. For the under-determined setting, the $p$-bridge shows a decreasing prediction trend of accuracy with respect to lowering of $k$ values. This can be interpreted as a trading off of accuracy with a higher sparseness of estimation.
	\item Sparseness of coefficients: an interesting observation from Fig.~\ref{fig_totalsum_coefficients} is the high agreement of the zero-value coefficients with the background regions of the summed digits for $p$-bridge. Moreover, the emphasized coefficients (for both positive and negative values) are seen to fall on unique regions for each digit in Fig.~\ref{fig_heatmap_mnist}. 
\end{itemize}

\section{Conclusion}

Pattern classification with compact representation is a key component in AI and machine learning. 
In this work, an analytic solution has been derived for bridge regression based on an approximation to the $\ell_p$-norm. The solution comes in primal form for over-determined systems and in dual form for under-determined systems. The solution in dual form, which does not find any precedents in the literature, is particularly useful for small sample size learning. These solution forms are extended for problems with multiple outputs. An algorithm has been implemented to pack the two solution forms into one framework for recognition applications. Several numerical examples and real-word data sets validated the usefulness of the algorithm for compressive applications. Interestingly, apart from the perspective of prediction accuracy, the digits recognition problems reveal linkage of those estimated coefficients to the informative pixels.

 
%



\end{document}